\documentclass[11pt]{article}
\usepackage{times}
\usepackage{helvet}
\usepackage{courier}

\usepackage{comment}

\makeatletter
\newcommand\footnoteref[1]{\protected@xdef\@thefnmark{\ref{#1}}\@footnotemark}
\makeatother

\usepackage{todonotes}
\usepackage{amsthm}
\usepackage{amsfonts}
\usepackage{amsmath}
\usepackage{amssymb}

\usepackage{multirow}
\usepackage{xspace}
\usepackage{lscape}

\usepackage{graphicx}
\newcommand\tabrotate[1]{\rotatebox{90}{#1\hspace{\tabcolsep}}}
\newcommand\rotfix[1][-.75\normalbaselineskip]{\hspace{#1}}
\newcommand{\highcell}{\rule{0cm}{0.35cm}}
\newcommand{\specialcell}[2][c]{%
  \begin{tabular}[#1]{@{}c@{}}#2\end{tabular}}

\usepackage[numbers,sort&compress]{natbib}
\usepackage[pagebackref]{hyperref}

\usepackage{comment}
\newtheorem{definition}{Definition}

\newtheorem{proposition}{Proposition}
\newtheorem{theorem}{Theorem}
\newtheorem{lemma}{Lemma}
\newtheorem{corollary}{Corollary}

\newcommand{\np}{\ensuremath{\mathrm{NP}}}
\newcommand{\para}{\ensuremath{\mathrm{Para}}}
\newcommand{\fpt}{\ensuremath{\mathrm{FPT}}}
\newcommand{\xp}{\ensuremath{\mathrm{XP}}}
\newcommand{\wone}{\ensuremath{\mathrm{W[1]}}}
\newcommand{\wtwo}{\ensuremath{\mathrm{W[2]}}}
\newcommand{\w}{\ensuremath{\mathrm{W}}}
\newcommand{\p}{\ensuremath{\mathrm{P}}}

\newcommand{\naturals}{\ensuremath{{\mathbb{N}}}}

\newcommand{\calR}{\ensuremath{{{\mathcal{R}}}}}
\newcommand{\calS}{\ensuremath{{{\mathcal{S}}}}}

\newcommand{\opt}{\ensuremath{{{\mathrm{OPT}}}}}
\newcommand{\pref}{\succ}

\newcommand{\pos}{\ensuremath{{{\mathrm{pos}}}}}

\newcommand{\shift}{\ensuremath{{{\mathrm{shift}}}}}

\newcommand{\revnot}[1]{\overleftarrow{#1}}

\newcommand{\probClique}{\textsc{Clique}\xspace}
\newcommand{\probSC}{\textsc{Set Cover}\xspace}

\newcommand{\probMCIS}{\textsc{Multicolored}\ \textsc{Independent Set}\xspace}

\newcommand{\ShiftBribery}{\textsc{Shift Bribery}\xspace}

\newcommand{\SNTV}{SNTV\xspace}
\newcommand{\Bloc}{Bloc\xspace}
\newcommand{\Borda}{Borda\xspace}
\newcommand{\kBorda}{$k$-Borda\xspace}
\newcommand{\ApprovalCC}{Approval-CC\xspace}
\newcommand{\BordaCC}{Borda-CC\xspace}

\newcommand{\GreedyBordaCC}{Greedy-Borda-CC\xspace}
\newcommand{\GreedyApprovalCC}{Greedy-Approval-CC\xspace}
\newcommand{\PTASCC}{PTAS-CC\xspace}

\title{Complexity of Shift Bribery in \\ Committee Elections\footnote{An extended abstract of this paper appeared
under the title ``Complexity of shift bribery in committee elections''
in the \emph{Proceedings of the 30th AAAI Conference on Artificial Intelligence (AAAI~'16)}, pages 2452--2458~\cite{preliminary}.}}
\author{Robert Bredereck$^1$, Piotr Faliszewski$^2$, Rolf Niedermeier$^1$, \\ Nimrod Talmon$^3$\thanks{Most of the work was done while the author was affiliated with TU Berlin.}\\ \ \\
  $^1$TU Berlin, Berlin, Germany\\
   \{robert.bredereck, rolf.niedermeier\}@tu-berlin.de\\
 $^2$AGH University of Science and Technology, Krakow, Poland\\
 faliszew@agh.edu.pl\\
 $^3$Ben-Gurion University, Be'er Sheva, Israel\\
 talmonn@bgu.ac.il}
\date{}

\usepackage{footnote}
\makesavenoteenv{tabular}

\begin{document}

\maketitle

\begin{abstract}
  Given an election, a preferred candidate~$p$, and a budget, the \textsc{Shift Bribery} problem
  asks whether~$p$ can win the election after shifting~$p$ higher in some voters' preference orders.
  Of course, shifting comes at a price (depending on the voter and on the extent of the shift)
  and one must not exceed the given budget.
  We study the (parameterized) computational complexity of \textsc{Shift Bribery}
  for multiwinner voting rules where winning the election means to be part of some winning committee.
  We focus on the well-established SNTV, Bloc, $k$-Borda,
  and Chamberlin-Courant rules, as well as on approximate variants of the
  Chamberlin-Courant rule, since the original rule is $\np$-hard to
  compute. We show that \textsc{Shift Bribery} tends to be harder in
  the multiwinner setting than in the single-winner one by showing
  settings where \textsc{Shift Bribery} is easy in the single-winner
  cases, but is hard (and hard to approximate) in the multiwinner
  ones.  Moreover, we show that the non-monotonicity of those rules
  which are based on approximation algorithms for the
  Chamberlin-Courant rule sometimes affects the complexity of
  \textsc{Shift Bribery}.
\end{abstract}

\section{Introduction}
We study the computational complexity of campaign management---modeled as the
\textsc{Shift Bribery} problem---for the case of multiwinner
elections. In the \textsc{Shift Bribery} problem we want to ensure
that our candidate is in a winning committee by convincing some of the
voters---at a given price---to rank him or her more favorably.  In
particular, this models campaigns based on direct meetings with
voters, in which the campaigner presents positive features of the
candidate he or she works for (see also the works of
Cary~\cite{Car11}, Magrino et al.~\cite{MRSW11}, Xia~\cite{Xia12}, and
Faliszewski at al.~\cite{FST17} for other interpretations of bribery
problems).  While the complexity of campaign management is relatively
well-studied for single-winner elections~\cite{FR15}, it has not been studied for
the multiwinner setting yet (there are, however, studies of
manipulation and control for multiwinner
elections~\cite{MPRZ08,AGGMMW15}).

The goal of a multiwinner election is to pick a committee of $k$
candidates based on the preferences of the voters. These $k$
candidates might, for example, form the country's next parliament, be
a group of people shortlisted for a job opening, or be a set of items
a company offers to its customers (see the overview of Faliszewski et
al.~\cite{fal-sko-sli-tal:b:multiwinner-voting} and the papers of Lu
and Boutilier~\cite{BL11}, Skowron et al.~\cite{SFL16}, and Elkind et
al.~\cite{EFSS14} for a varied description of applications of
multiwinner voting).
Since the election results can affect the voters and the candidates
quite significantly, we expect that they will run campaigns to achieve
the most desirable results: a person running for parliament would
want to promote her or his political platform; a job candidate would want to
convince the HR department of her or his qualities. 

We study the standard, ordinal model of voting, where each voter
ranks the candidates from the one he or she likes best to the one he
or she likes least. 
We focus on rules that are based either on the Borda scores of the
candidates or on their $t$-Approval scores.
Briefly put, if we have $m$ candidates, then a voter gives Borda score $m-1$
to his or her most preferred candidate, score $m-2$ to the next one, and so on;
a voter gives $t$-Approval score $1$ to each of his or her top-$t$ candidates
and score $0$ to the other ones.

The most basic multiwinner rules simply pick $k$ candidates with the
highest scores (for example, SNTV uses $1$-Approval scores, Bloc uses
$k$-Approval scores, and $k$-Borda uses Borda scores). While such
rules may be good for shortlisting tasks, they do not seem to perform
well for cases where the committee needs to be varied (or represent
the voters proportionally; see the overview of Faliszewski et
al.~\cite{fal-sko-sli-tal:b:multiwinner-voting} and the work of Elkind
et al.~\cite{EFSS14}). In this case, we may prefer other rules, such
as those in the Chamberlin-Courant family of rules~\cite{CC83}, which
try to ensure that every voter is represented well by some member of
the committee (see Section~\ref{sec:prelim} for an exact definition),
or the STV rule.

Unfortunately, while the winners of SNTV, Bloc, and $k$-Borda rules
are polynomial-time computable, this is not the case for the
Chamberlin-Courant rules (Procaccia et al.~\cite{PRZ08} and Lu
and Boutilier~\cite{BL11} show $\np$-hardness). We deal with this
problem in two ways.  First, there are exact FPT algorithms for computing
Chamberlin-Courant winners (for example, for the case of few
voters). Second, there are good polynomial-time approximation
algorithms 
(due to Lu and Boutilier~\cite{BL11} and Skowron et
al.~\cite{SFS15}). Following 
Caragiannis et al.~\cite{CKKP14} and Elkind et al.~\cite{EFSS14}, we
consider these approximation algorithms as voting rules in their own
right (
societies may 
use them in place of the original, hard-to-compute ones).

The idea of the \textsc{Shift Bribery} problem is as follows. We are
given an election and a preferred candidate $p$, and we want to ensure
that $p$ is a winner (in our case, is a member of a winning committee)
by shifting him or her forward in some of the votes, at an appropriate
cost, without exceeding a given budget.  The costs of shifting $p$
correspond to investing resources into convincing the voters that our
candidate is of high quality. For example, if a company is choosing
which of its products to continue selling, the manager responsible for
a given product may wish to prepare a demonstration for the company's
higher management. Similarly, a person running for parliament would
invest money into meetings with the voters, appropriate leaflets, and
so on. Thus, we view \textsc{Shift Bribery} as a model of (a type of)
campaign management.
Nonetheless, there are also other appealing interpretations of the
bribery problems. For example, Cary~\cite{Car11}, Magrino et
al.~\cite{MRSW11}, and Xia~\cite{Xia12} studied the margin of victory
problem (which is a form of destructive bribery), where the goal is to
ensure that a given candidate does not win by changing as few votes as
possible. The fewer votes need to be changed, the more likely it is
that a given election was tampered with. Similarly, Faliszewski et
al.~\cite{FST17} suggested that the amount of bribery needed to ensure
that a given candidate wins is a good measure of how well this
candidate performed in the election; the fewer changes in the votes
are necessary, the closer a given candidate was to victory (indeed,
Faliszewski et al.~\cite{FST17} argue that this measure might be more
appealing than the candidates' scores). This measure-of-success
interpretation applies to our work as well.

\textsc{Shift Bribery} was introduced by Elkind et
al.~\cite{EFS09,EF10}, and since then a number of other researchers
studied both \textsc{Shift Bribery}~(e.g., Schlotter et
al.~\cite{EFS17}, Bredereck et al.~\cite{BCFNN16,BFNT16}, Kaczmarczyk
and Faliszewski~\cite{kac-fal:c:dsb} and Maushagen et
al.~\cite{MNRS18}), and related campaign management problems (e.g.,
Dorn and Schlotter~\cite{DS12}, Baumeister et al.~\cite{BFLR12},
Faliszewski et al.~\cite{FRRS15}, and Knop et al.~\cite{KMM17}).
Naturally, the problem also resembles other bribery problems, such as
the original bribery problem of Faliszewski et al.~\cite{FHH09} or
those studied by Mattei et al.~\cite{MPRV12} and Mattei, Goldsmith,
and Klapper~\cite{MGK12}. We point the reader to the overview of
Faliszewski and Rothe~\cite{FR15} for more details and references.

For single-winner elections, \textsc{Shift Bribery} is a relatively easy problem.
Specifically, it is polynomial-time solvable for the $t$-Approval rules.
For the Borda rule, for which it is $\np$-hard, there is a good polynomial-time approximation
algorithm~\cite{EF10} and there are exact $\fpt$ algorithms~\cite{BCFNN16}.
In the multiwinner setting the situation is quite different.
The main findings of our research are as follows (see also \autoref{table:results} in \autoref{sec:shift-bribery}):

\begin{enumerate}
\item The computational complexity of \textsc{Shift Bribery} for multiwinner rules
  strongly depends on the setting. In general, for the cases of few
  candidates we find exact $\fpt$ algorithms while for the cases where the
  preferred candidate is shifted by few positions only we find
  hardness results (even though these cases are often easy in the
  single-winner setting). 

\item The computational complexity for the case of few voters most
  strongly depends on the underlying scoring rule. Generally, for the
  rules based on $t$-Approval scores the complexity of \textsc{Shift
    Bribery} tends to be lower than for analogous rules based on
  Borda scores.
  
\end{enumerate}

\noindent
We did not study multiwinner rules such as the STV rule, the Monroe
rule~\cite{Mon95}, or the rules for the approval elections (see, e.g., the works
of Brams and Kilgour~\cite{BK14}, Aziz et
al.~\cite{ABCEFW15,AGGMMW15}, and Lackner and
Skowron~\cite{lac-sko:c:approval-thiele}), in order to keep our set of
rules small, while being able to compare our results to those for the
single-winner setting (however, we mention that Maushagen et
al.~\cite{MNRS18} considered \textsc{Shift Bribery} for round-based
rules, including STV, and Faliszewski et al.~\cite{FST17} considered
problems analogous to \textsc{Shift Bribery} for the case of
approval-based multiwinner rules).

\section{Preliminaries}\label{sec:prelim}

\noindent\textbf{Elections and Voting Rules.}\quad For each integer $n$, we set
$[n] := \{1, \ldots, n\}$.  An election~$E = (C,V)$ consists of a set
of candidates $C = \{c_1, \ldots, c_m\}$ and a collection of voters $V
= (v_1, \ldots, v_n)$. Each voter $v$ is associated with a preference
order, i.e., with a ranking of the candidates in decreasing order
of appreciation by the voter.  For example, if $C = \{c_1, c_2, c_3\}$,
then by writing $v \colon c_1 \pref c_2 \pref c_3$ we mean that
$v$~likes~$c_1$ best, then~$c_2$, and then~$c_3$. We write~$\pos_v(c)$
to denote the position of candidate~$c$ in voter~$v$'s preference order
(e.g., in the preceding example we would have~$\pos_v(c_1) = 1$).
When we write a subset~$A \subseteq C$ of candidates in a description
of a preference order, we mean listing all members of~$A$ in some
fixed, easily computable order. If we put~$\revnot{A}$ in a preference
order, then we mean listing members of~$A$ in the reverse of this fixed
order.

Let $E = (C,V)$ be an election with $m$~candidates and $n$~voters.
The Borda score of candidate~$c$ in the vote of~$v$, $v \in V$, is
$\beta_v(c) = m-\pos_v(c)$.  The Borda score of~$c$ in the election~$E$ is
$\beta_E(c) = \sum_{v \in V}\beta_v(c)$. The single-winner Borda rule
elects the candidate with the highest Borda score (if there are
several such candidates, they tie as winners).
For each $t \in [m]$, we define the $t$-Approval score as follows: for
a candidate~$c$ and voter~$v$, $\alpha^t_v(c) = 1$ if $v$~ranks~$c$
among the top~$t$ positions and otherwise it is $0$; we set $\alpha^t_E(c)
= \sum_{v \in V}\alpha^t_v(c)$. We define the single-winner
$t$-Approval rule analogously to the Borda rule.

A multiwinner voting rule~$\calR$ is a function that, given an
election~$E=(C,V)$ and an integer $k\in [|C|]$, outputs a set
$\calR(E,k)$ of $k$-element subsets of~$C$. Each size-$k$ subset of
$C$ is called a \emph{committee} and each member of $\calR(E,k)$ is 
called a \emph{winning committee}. 

The most natural task that arises when considering (multiwinner) voting
rules is the task of deciding whether a given candidate is among the winners
(resp.\ is part of some winning committee).
We will refer to this task as the \textsc{Winner Determination} problem.
Sometimes, winner determination procedures studied in the literature
consider slightly different goals (e.g.\ computing the score of a winning committee).
However, all polynomial-time, $\fpt$, and $\xp$ winner determination procedures
for the rules we study in this paper can be modified to 
solve \textsc{Winner Determination}.

We consider the following 
rules
(below, $E = (C,V)$ is an election and $k$ is the
committee size):

 \begin{enumerate}
 \item 
  \emph{\SNTV}, \emph{\Bloc}, and \emph{\kBorda} compute the score of each candidate
  and output the committee of $k$~candidates with the highest scores
  (or all such committees, if there are several).  \SNTV and
  \Bloc use, respectively, $1$-Approval and $k$-Approval scores, while \kBorda uses
  Borda scores.
  For these rules winners can be computed in polynomial time.\footnote{
    There may be exponentially many winning committees, but it is
    easy to compute their score and to check 
    for a subset of candidates if it can be extended to a winning committee.}

 \item
  Under the \emph{Chamberlin-Courant rules} (the CC rules), for a
  committee $S$, a candidate~$c \in S$ is a representative of those
  voters that rank $c$ highest among the members of~$S$.  The
  score of a committee is the sum of the scores that the voters give
  to their representatives (highest-scoring committees
  win); \emph{Borda-CC} uses Borda scores, \emph{$t$-Approval-CC}
  uses $t$-Approval scores.
  \textsc{Winner Determination} for CC rules is $\np$-hard~\cite{PRZ08,BL11},
  but is in~$\fpt$ when parameterized by the number of voters or
  candidates~\cite{BSU13}.

 \item
  \emph{Greedy-Borda-CC} is a $(1-\frac{1}{e})$-approximation
  algorithm for the Borda-CC rule, due to Lu and
  Boutilier~\cite{BL11}.  (The approximation is in the sense that
  the score of the committee output by the algorithm is at least a
  $1-\frac{1}{e}$ fraction of the score of the winning committee
  under Borda-CC.)  The algorithm starts with an empty set~$W$ and
  executes $k$~iterations, in each one adding to~$W$ the candidate~$c$
  that maximizes the Borda-CC score of~$W \cup \{c\}$.\footnote{If
    there is a tie between several candidates, then we assume that the
    algorithm breaks it according to a prespecified
    order. 
  }
  For example, it always picks a Borda winner in the first iteration.
  Greedy-Borda-CC always outputs a unique winning committee.

\item  \emph{Greedy-Approval-CC} works in the same way as Greedy-Borda-CC, but
  uses $t$-Approval scores instead of Borda scores.
  It is a $(1-\frac{1}{e})$-approximation
  algorithm  for  $t$-Approval-CC.
We refer to $t$-Approval-Greedy-CC for $t = \lceil \frac{m
 \cdot \mathrm{w}(k)}{k}\rceil$ (where $\mathrm{w}$ is Lambert's W
function; $\mathrm{w}(k)$ is $O(\log k)$) as \emph{PTAS-CC}; it is the main
part of Skowron et al.'s~\cite{SFS15} polynomial-time
approximation scheme for Borda-CC.
 \end{enumerate}

\noindent\textbf{Parameterized Complexity.}\quad In a parameterized problem, we
declare some part of the input (e.g., the number of voters) as the
\emph{parameter}. A parameterized problem is \emph{fixed-parameter tractable}
(is in $\fpt$) if there is an algorithm that solves it in
$f(\rho) \cdot |I|^{O(1)}$ time, where $|I|$ is the size of a given
instance encoding, $\rho$ is the value of the parameter, and $f$ is
some computable function. There is a hierarchy of classes of hard
parameterized problems,
$\fpt \subseteq \wone \subseteq \wtwo \subseteq \cdots \subseteq \xp$.
It is widely believed that if a problem is hard for one of the
$\mathrm{W[\cdot]}$ classes, then it is not in $\fpt$. The notions of
hardness and completeness for parameterized classes are defined
through \emph{parameterized reductions}. For this paper, it suffices to use
standard polynomial-time many-one reductions that guarantee that the
value of the parameter in the problem we reduce to exclusively depends
on the value of the parameter of the problem we reduce from.

If a parameterized problem can be solved in polynomial time under the
assumption that the parameter is constant, then we say that it is in~$\xp$.
Recall that membership in $\fpt$ additionally requires that the degree
of the polynomial is a constant independent from the parameter.
If a problem is $\np$-hard even for some constant value of the parameter,
then we say that it is para-$\np$-hard.

For details on parameterized complexity, we point to
the books of Cygan et al.~\cite{CFKLMPPS15},
Downey and Fellows~\cite{DF13}, Flum and Grohe~\cite{FG06},
and Niedermeier~\cite{Nie06}.

\section{Shift Bribery}\label{sec:shift-bribery}

Let $\calR$ be a multiwinner rule. In the
$\calR$-\textsc{Shift Bribery} problem we are given an election $E =
(C,V)$ with $m$~candidates and $n$~voters, a preferred candidate~$p$,
a committee size~$k$, voter price functions (see below), and an
integer~$B$, the budget. The goal is to ensure that $p$ belongs to at
least one winning committee (according to the
rule~$\calR$),\footnote{Our approach 
  is a natural extension of the non-unique winner model from the world of
  single-winner rules.
  Naturally, one might alternatively require that $p$ is a
  member of all winning committees or put an even more demanding goal
  that would involve other candidates.}
and to achieve this goal we are allowed to
shift~$p$ forward in the preference orders of the voters. However,
each voter~$v$ has a price function $\pi_v \colon [m] 
\rightarrow \naturals$, and if we shift~$p$ by $i$~positions forward
in the vote of~$v$, then we have to pay $\pi_v(i)$. We assume that the
price functions are nondecreasing (i.e., it cannot cost less to shift
our candidate farther than to shift her or him nearer) and that the
cost of not shifting~$p$ is zero (i.e., $\pi_v(0) = 0$ for each~$v$).
Bredereck et al.~\cite{BCFNN16} have considered several different
families of price functions. In this paper we focus on two of them:
unit price functions, where for each voter~$v$ it holds that $\pi_v(i)
= i$, and all-or-nothing price functions, where for each voter~$v$ it
holds that $\pi_v(i) = q_v$ for each $i > 0$ (where $q_v$ is some
voter-dependent value) and $\pi_v(0) = 0$.

A shift action is a vector $(s_1, \ldots, s_n)$ of natural numbers
that for each voter specify by how many positions to shift $p$.  If
$\vec{s} = (s_1, \ldots, s_n)$ is a shift action, then we write
$\shift(E,\vec{s})$ to denote the election obtained from~$E$ by
shifting~$p$ an appropriate number of positions forward in each
vote. If $\Pi = (\pi_1, \ldots, \pi_n)$ are the price functions of the
$n$~voters, then we write $\Pi(\vec{s}) = \sum_{i=1}^n\pi_i(s_i)$ to
denote the total cost of applying $\vec{s}$. For a shift action
$\vec{s}$, we define $\#\vec{s} = \sum_{i=1}^n s_i$ and 
we call it the number of unit shifts in~$\vec{s}$. 

Formally, we define $\calR$-\textsc{Shift Bribery} as follows.
\begin{definition}
  Let $\calR$ be a multiwinner voting rule.
  An instance~$I$ of $\calR$-\textsc{Shift Bribery} consists of an election $E = (C,V)$,
  a preferred candidate $p \in C$, a committee size~$k$, a collection
  $\Pi = (\pi_1, \ldots, \pi_n)$ of price functions for the
  voters, and an integer $B$, the budget.
  We ask whether there is a
  shift action $\vec{s} = (s_1, \ldots, s_n)$ such that:
  \begin{enumerate}
   \item $\Pi(\vec{s}) \leq B$, and
   \item there is a committee $W \in
         \calR(\shift(E,\vec{s}),k)$ such that $p \in W$.
  \end{enumerate}
  We refer to such a shift action as a \emph{successful shift action};
  we write $\opt(I)$ to denote the cost of the least expensive successful
  shift action.
\end{definition}

Following Bredereck et al.~\cite{BCFNN16}, we consider
the most natural parameterizations
by the number~$n$ of the voters, by the number~$m$ of the candidates, 
and by the minimum number~$s$ of unit shifts in a successful shift action.
We summarize our results, as well as some previously known ones, in
\autoref{table:results}.
The reminder of this paper is structured as follows.  In
\autoref{sec:general_results}, we present findings applying to the
multiwinner context as a whole.  In \autoref{sec:topk-results}, we
present specific results for the voting rules SNTV, Bloc, and
$k$-Borda.  In \autoref{sec:CC-results}, we present our results for
Chamberlin-Courant rules and their approximate variants.  We conclude
in \autoref{sec:conclusion} with a discussion and an outlook.

\begin{table*}[t]
  {\centering
\medskip
\resizebox{\linewidth}{!}{
\begin{tabular}{l | c || c || c | c | c |}
 \hline
  \multicolumn{2}{|c||}{\multirow{2}{*}{voting rule $\calR$}} & \highcell $\calR$-\textsc{Winner} & \multicolumn{3}{c|}{$\calR$-\ShiftBribery} \\
 \cline{4-6}
  \multicolumn{2}{|c||}{}                                     & \highcell \textsc{Determination}                 & \#candidates ($m$)                         & \#voters ($n$)                                     & \#shifts ($s$) \\
 \hline
 \multicolumn{1}{l}{}\multirow{5}{*}{\tabrotate{\rotfix \specialcell{single\\winner}}} & \multicolumn{5}{l}{}\\
 \cline{2-6}
 \highcell & $t$-Approval              & \multirow{3}{*}{\p{}$^\bigstar$}  & \multicolumn{3}{c|}{\p{}$^\triangledown$} \\
 \cline{2-2} \cline{4-6}    
 \highcell & \multirow{2}{*}{Borda}    &                                     & \multirow{2}{*}{$\fpt${}$^\diamondsuit$} & \fpt(0/1-pr.){}, \fpt-AS{}$^\diamondsuit$, & \multirow{2}{*}{$\fpt$$^\diamondsuit$} \\
  &                                          &                                     &                                            & and $\w[1]$-h (Thm.~\ref{borda_wrt_n})   & ~ \\
 \cline{2-6}
 \multicolumn{6}{l}{}\\
 \cline{2-6}
 \multirow{8}{*}{\tabrotate{\rotfix \specialcell{\multirow{2}{*}{multi winner}\\}}}
 \highcell & SNTV                       &  \multirow{4}{*}{\p{}$^\bigstar$} & \multicolumn{3}{c|}{\multirow{2}{*}{\p{} (Prop.~\ref{prop:sntvbloc})}}  \\
 \cline{2-2}   
 \highcell & Bloc                      &           ~                          & \multicolumn{3}{c|}{\multirow{2}{*}{~}}  \\
 \cline{2-2} \cline{4-6}
 \highcell & $k$-Borda &           ~                          & \multirow{8}{*}{$\fpt$ (Thm.~\ref{thm:m})} & \fpt(0/1-pr.) (Prop.~\ref{prop:all-or-nothing-n}), & \wone-h (Thm.~\ref{kborda_wrt_s})                     \\
 \cline{2-3} \cline{6-6}
 \highcell & \multirow{2}{*}{Borda-CC} & $\np$-h$^\spadesuit$,              &       ~                                    & \fpt-AS (Prop.~\ref{prop:fpt-as}), and               & \multirow{3}{*}{\para-\np-h$^\spadesuit$}           \\
  & ~                         & $\fpt(n)$$^\heartsuit$, and        &       ~                                    & $\w[1]$-h (Cor.~\ref{cor:kBorda-n}+Cor.~\ref{cor:BordaCC-n})                 & ~                                                     \\
 \cline{2-2} \cline{5-5}
 \highcell & Approval-CC               & $\fpt(m)$$^\heartsuit$             &       ~                                    & \multirow{3}{*}{\fpt{} (Prop.~\ref{prop:greedy-approval-cc-n})} & ~                                          \\
 \cline{2-3} \cline{6-6}
 \highcell & Greedy-Approval-CC        & \multirow{3}{*}{\p{}$^\bigstar$}   &       ~                                    & ~                                                  & \multirow{3}{*}{\wtwo-h (Thm.~\ref{W2h-unitshifts})}  \\
 \cline{2-2}   
 \highcell & PTAS-CC                   &           ~                          &       ~                                    & ~                                                  & ~                                                     \\
 \cline{2-2} \cline{5-5}
 \highcell & Greedy-Borda-CC &      ~                          &       ~                                    & $\w[1]$-h (Cor.~\ref{cor:BordaCC-n})                 & ~                                                     \\
 \cline{2-6}
\end{tabular}
}
}

  \caption{
    Overview of our complexity results for the \ShiftBribery problem
    (for reference, we also mention the complexity of the \textsc{Winner Determination} problem).
    The results in each cell apply to all voting rules listed in the leftmost column which span the height of the cell.
    All results are for the case of unit price functions, with the exceptions of those marked as \fpt(0/1-pr.), which are for all-or-nothing
    price functions (many other results extend to other price functions, but we do not list them here).
    \fpt-AS stands for $\fpt$ approximation scheme (see \autoref{prop:fpt-as}).
    Note that all variants which are $\w$[$\cdot$]-hard are also in~$\xp$.
    Results marked
    by $\triangledown$ follow from the work of Elkind et al.~\cite{EFS09},
    by $\diamondsuit$ follow from the work of Bredereck et al.~\cite{BCFNN16},
    by $\spadesuit$ follow from the works of Procaccia et al.~\cite{PRZ08} and Lu and Boutilier~\cite{BL11},
    by $\heartsuit$ follow from the work of Betzler et al.~\cite{BSU13},
    and by $\bigstar$ are folklore results. 
  }\label{table:results}

\end{table*}

\section{General Results}
\label{sec:general_results}

We start our discussion by providing several results that either apply
to whole classes of multiwinner rules (including many of those that we
focus on) or that are proven using general, easily adaptable techniques.
These results form a baseline for our research regarding specific
rules.

First, we note that for each of the rules that we study, \textsc{Shift
  Bribery} with unit price functions is in $\fpt$ when parameterized
by the number of candidates. This result follows by applying the
standard technique of modeling the problem through an integer linear
program and invoking Lenstra's theorem~\cite{Len83}. We conjecture that,
using the MILP technique of Bredereck et al.~\cite{BFNST15}, or the
more general toolbox of $n$-fold integer
programming~\cite{hem-onn-rom:j:n-fold-ip} (see the work of Knop et
al.~\cite{KMM17} for an application of $n$-fold IPs regarding other
bribery problems), it is also possible to generalize this result to
all-or-nothing price functions (or even to general price functions).

Note that the following theorem does not mention SNTV and Bloc since,
as we will see in the next section, for them the problem is even
polynomial-time solvable.

\newcommand{\thmm}{Parameterized by the number of
  candidates, \textsc{Shift Bribery} with unit prices is in $\fpt$ for
  $k$-Borda, Approval-CC, Borda-CC, Greedy-Approval-CC, PTAS-CC, and
  Greedy-Borda-CC.}

\begin{theorem}\label{thm:m}
  \thmm
\end{theorem}

In order to prove \autoref{thm:m}, we introduce an algorithmic scheme similar
to that of \citet{DS12} for single-winner \textsc{Swap Bribery}.
We will make use of the fact that integer linear programs (ILPs)
can be solved in FPT time with respect to the number of (integer) variables
(following a famous result by \citet{Len83} which was later improved
by \citet{Kan87} and by \citet{FT87}).
We first introduce the scheme and the basic ILP formulation.
Then, we show how to extend the ILP so that the algorithmic scheme works
for \kBorda (by proving \autoref{lem:m:kBorda}),
for \ApprovalCC and \BordaCC (by proving \autoref{lem:m:CCs}),
and for \GreedyApprovalCC, \PTASCC, and \GreedyBordaCC (by proving \autoref{lem:m:Greedys}).

The idea of the scheme is to guess the members of the winning
committee $W \subseteq C$, $|W|=k$, $p \in W$, and to verify the guess
by solving an ILP.  More precisely, we try all possible winning
committees in the outer loop of our algorithm and call the
corresponding ILP for each of the (less than~$2^m$) potential winning
committees that contain~$p$.  For the round-based rules
(\GreedyApprovalCC, \PTASCC, and \GreedyBordaCC) we furthermore guess
a function $w: [k] \rightarrow W$ mapping each ``position'' in the
committee to a specific candidate from~$W$.  This allows us to specify
in which round each member joined the committee according to the
round-based rules and can be realized with an additional factor of
$k!\le m!$ in the running time.  For each $j \in [k]$, let~$W^j$
denote the set of the first~$j$ members according to the function~$w$,
that is, $W^j=\{w(j') \mid 1\le j' \le j\}$.

There are $m!$ different preference orders and, by ordering them
arbitrarily, we can speak of the $i$-th preference order for
each $i \in [m!]$. Below we describe the main components of our ILPs.

For each $i \in [m!]$ and $j \in [m!]$, we introduce an integer
variable $S_{i,j}$ which represents the number of voters who
originally have the $i$-th preference order, but who will have
the $j$-th one after the bribery.
We add the following constraints for each $i \in [m!]$, ensuring that
each original vote is turned into exactly one bribed vote (by $n_i$ we
mean the number of voters who, prior to the bribery, have the $i$-th
preference order):
  \[
  \sum_{j \in [m!]} S_{i,j} = n_i.
  \]
  Then, we add the following constraint, ensuring that the cost of our
  bribery action does not exceed the budget $B$ (by $\textrm{cost}(i,j)$
  we mean the number of unit shifts necessary to transform the $i$-th
  preference order into the $j$-th one, or $B+1$ if such a
  transformation is impossible):
  \[
  \sum_{i \in [m!], j \in [m!]} S_{i,j} \cdot \textrm{cost}(i,j) \leq
  B.
  \]
  For each $j \in [m!]$, we introduce an integer variable $N_{j}$
  which represents the number of voters who have the $j$-th
  preference order after the bribery. To ensure that these variable
  have the correct values, for each $j \in [m!]$ we introduce the following constraint:
  \[
    N_{j} = \sum_{i \in [m!]} S_{i,j}. 
  \]
This describes the basic ILP which will be extended in
the proofs of the following lemmas.

\begin{lemma}\label{lem:m:kBorda}
  Parameterized by the number $m$ of candidates,
  \kBorda \ShiftBribery is in FPT.
\end{lemma}

\begin{proof}
  To ensure that~$p$ is a member of the \kBorda winning committee,
  we have to guarantee that only the other members of the winning
  committee may have larger Borda scores than~$p$.
  Hence, for each candidate $c \notin W$, we add the following
  constraint to the base ILP (by $\beta_i(c)$ we mean the Borda score
  that candidate~$c$ receives in the $i$-th preference order):
  \[
  \sum_{i \in [m!]} N_{i} \cdot \beta_i(p) \geq \sum_{i \in
    [m!]} N_{i} \cdot \beta_i(c).
  \]
  This finishes the description of the extended ILP.
\end{proof}

\begin{lemma}\label{lem:m:CCs}
  Parameterized by the number~$m$ of candidates, both \ApprovalCC
  \ShiftBribery and \BordaCC \ShiftBribery are in $\fpt$.
\end{lemma}

\begin{proof}
  To ensure that~$p$ is a member of the \ApprovalCC (respectively,
  \BordaCC) winning committee~$W$, we have to guarantee that no other
  committee has a larger Approval score (respectively, Borda score)
  than our guessed committee~$W$.
  Hence, for each other committee $W'$, we add the following
  constraint to the base ILP, ensuring that in the bribed election,
  the score of~$W$ (based on the $N_{i}$~variables) is at least as
  high as the score of $W'$ (by $\phi(i,X)$ we mean the score
  assigned by a voter with the $i$-th preference order to committee
  $X$):
  \[
    \sum_{i \in [m!]} \phi(i,W) \cdot N_i \geq \sum_{i \in [m!]} \phi(i,W') \cdot N_i.
  \]
  Note that the values $\phi(i,W)$ and $\phi(i,W')$ can be precomputed in
  polynomial time and are constants from the point of view of the ILP.
  This finishes the description of the extended ILP.
\end{proof}

\begin{lemma}\label{lem:m:Greedys}
  Parameterized by the number~$m$ of candidates, \ShiftBribery is in
  $\fpt$ for \GreedyApprovalCC, \PTASCC, and \GreedyBordaCC.
\end{lemma}

\begin{proof}
  Since \PTASCC is a special case of \GreedyApprovalCC, it suffices to
  describe the extension of the ILP for \GreedyApprovalCC and for
  \GreedyBordaCC.
 
  To ensure that~$p$ is a member of the winning committee~$W$ for
  \GreedyApprovalCC (respectively, \GreedyBordaCC), we have to
  guarantee that the candidate~$w(j)$ (which joins the committee in
  the $j$-th round) maximizes the score of the committee, among all
  the candidates that can be added in the $j$-th round.
  Hence, for each round~$j$ and each $c \in C \setminus W^j$, we add
  the following constraint to the basic ILP, ensuring that in the
  bribed election, the score of~$W^j$ is at least as large as the
  score of $W^{j-1} \cup \{c\}$ (as in the previous lemma, by
  $\phi(i,X)$ we mean the score assigned by a voter with the $i$-th
  preference order to committee $X$):
  \[
    \sum_{i \in [m!]} \phi(i,W^j) \cdot N_i \geq \sum_{i \in [m!]} \phi(i,W^{j-1} \cup \{c\}) \cdot N_i.
  \]
  This finishes the description of the extended ILP.
\end{proof}

The proofs of \autoref{lem:m:kBorda}, \autoref{lem:m:CCs}, and \autoref{lem:m:Greedys} complete
the proof of \autoref{thm:m}.

As the second general result, we note that for the parameterization by the number of voters
we can provide a strong, general $\fpt$ approximation scheme for
\emph{candidate-monotone} rules. Candidate monotonicity, a notion
introduced by Elkind et al.~\cite{EFSS17}, requires that if a
member of a winning committee is shifted forward in some vote, then
this candidate still belongs to some (possibly different) winning
committee.

\begin{proposition}\label{prop:fpt-as}
  Consider parameterization by the number of voters.
  Let $\calR$ be a candidate-monotone multiwinner rule with an
  $\fpt$ algorithm for \textsc{Winner Determination}.
  Then, for every positive constant number~$\varepsilon$
  there is an $\fpt$ algorithm that, given an instance~$I$
  of~$\calR$-\textsc{Shift Bribery} (for arbitrary price functions),
  outputs a successful shift action $\vec{s}$ 
  with cost at most $(1+\varepsilon)\opt(I)$.
\end{proposition}

\begin{proof}
  Bredereck et al.~\cite{BCFNN16} show an $\fpt$ algorithm
  (parameterized by the number of voters) that, given an instance~$I$
  of \textsc{Shift Bribery} and a positive value $\varepsilon$, for
  each possible shift action $\vec{s} = (s_1, \ldots, s_n)$ tries a
  shift action $\vec{s}' = (s'_1, \ldots, s'_n)$ such that for each $i
  \in [n]$ we have $s'_i \geq s_i$, and the cost of $\vec{s}'$ is at
  most $(1+\varepsilon)$ greater than that of $\vec{s}$.
  This algorithm also works for multiwinner rules.
\end{proof}

Among the rules considered in this work, only \GreedyBordaCC, \GreedyApprovalCC, and \PTASCC
are not candidate-monotone (see the work of Elkind et
al.~\cite{EFSS17} for the argument regarding Greedy-Borda-CC).
Thus, the above result applies to all remaining rules.

For the case of all-or-nothing prices, we can strengthen the above
result to an exact $\fpt$ algorithm.

\begin{proposition}\label{prop:all-or-nothing-n}
  Consider parameterization by the number of voters.  Let $\calR$ be a \linebreak
  candidate-monotone multiwinner rule with an $\fpt$ algorithm for
  \textsc{Winner Determination}.
  Then, there is an $\fpt$ algorithm for
  $\calR$-\textsc{Shift Bribery} with all-or-nothing price functions.
\end{proposition}

\begin{proof}
   Since $\calR$ is candidate-monotone and we have all-or-nothing
   prices, for every vote where we shift the candidate~$p$
   forward, we can shift~$p$ to the top. 
   In effect, it suffices to try all subsets of voters: For each subset
   check whether shifting~$p$ forward in each vote from the subset
   ensures the victory of~$p$ without exceeding the budget.
\end{proof}

 Using a very similar approach, we can solve \textsc{Shift Bribery} for
 those of our rules which are based on approval scores, even for
 arbitrary price functions (even the round-based ones). The trick is
 that, with approval scores, for each voter we either shift our
 candidate right to the lowest approved position or we do not shift him
 or her at all. Thus, again, trying all subsets of voters suffices.

\begin{proposition}\label{prop:greedy-approval-cc-n}
  There is an $\fpt$ algorithm for \textsc{Shift Bribery} under
  Approval-CC, Greedy-Approval-CC, and PTAS-CC, for the
  parameterization by the number of voters and for arbitrary price
  functions.
\end{proposition}

Finally, using smart brute-force, we provide $\xp$
algorithms for \textsc{Shift Bribery} parameterized either by the
number of voters or the number of unit shifts (for rules that can be
efficiently computed in the given setting).

\newcommand{\thmxpn}{Consider parameterization by the number of voters.
  Then, for every multiwinner rule with an $\xp$ algorithm for
  \textsc{Winner Determination}, there is an $\xp$ algorithm for
  \textsc{Shift Bribery} and arbitrary price functions.}

\begin{proposition}\label{prop:xpn}
  \thmxpn
\end{proposition}

\begin{proof}
  For each voter, we guess the amount which the preferred candidate
  is shifted by.  Since the maximum amount is $m$, and we have $n$
  voters, we have $O(m^n)$ possibilities to check.  For each
  possibility we check if the preferred candidate is a member of a
  winning committee in $\xp$ time.
\end{proof}

\newcommand{\thmxps}{Consider parameterization by the number of unit shifts.
  Then, for every multiwinner rule with a polynomial-time algorithm
  for \textsc{Winner Determination}, there is an $\xp$ algorithm for
  \textsc{Shift Bribery} and arbitrary price functions.}

\begin{proposition}\label{prop:xps}
  \thmxps
\end{proposition}

\begin{proof}
  The idea of the proof is similar to that behind \autoref{prop:xpn}.
  Let $s$ be the number of unit shifts that we can perform and let $n$
  be the number of voters.  We can view a solution as a vector of
  length at most $s$, where an entry in the $i$-th position is the
  name of the voter on whose preference order we apply the $i$-th unit
  shift. We try all $O(n^s)$ such vectors and for each we test if the
  shift action it defines is within the budget and ensures that the
  preferred candidate is in the winning committee.
\end{proof}

\section{SNTV, Bloc, and \textit{k}-Borda}
\label{sec:topk-results}

We now move on to results specific 
to
the voting rules SNTV, Bloc, and $k$-Borda. These rules 
pick $k$
candidates with the highest $1$-Approval, $k$-Approval, and Borda
scores, respectively, and, so, one might suspect that the efficient
algorithms for corresponding single-winner rules would translate to
the multiwinner setting. While this is the case for SNTV and Bloc, for
$k$-Borda the situation is more intricate. 
As a side effect of our research, we resolve the complexity of
Borda-\textsc{Shift Bribery} parametrized by the number of voters,
which was left open by Bredereck et al.~\cite{BCFNN14}.\footnote{\label{BCFNN14note}In fact, the
long version~\cite{BCFNN16} of Bredereck et al.~\cite{BCFNN14} 
refers to the short version of this work~\cite{preliminary} as resolving this open question.}

We first show that \ShiftBribery is polynomial-time solvable for \SNTV
and \Bloc. 
\begin{comment}  
The reason is that, for these rules, for each voter
shifting $p$ to a specified position (the first position, for \SNTV;
the $k$th position, for \Bloc) results, besides improving the score
of the preferred candidate~$p$, in decreasing the score of exactly one other candidate.
\end{comment}
Briefly put, the idea is to guess the final score of the preferred
candidate and to compute the set of candidates that have higher
scores. Then, given committee size $k$, it is easy to compute the
cheapest way to ensure that all but $k-1$ of these candidates have
smaller score than the guessed score of~$p$, while ensuring that $p$
indeed obtains this guessed score.
We rely on the fact that under both rules and for each vote it
suffices to consider only one possible shift action, either shifting
the preferred candidate to the top of the vote (for the case of SNTV)
or shifting the preferred candidate to the first approved position
(for the case of Bloc).

\newcommand{\thmsntvbloc}{SNTV-\ShiftBribery and Bloc-\ShiftBribery
  are both polynomial-time solvable (for arbitrary price functions).  }

\begin{proposition}\label{prop:sntvbloc}
  \thmsntvbloc
\end{proposition}

\begin{proof}
  We use the same algorithm for both SNTV and Bloc. Consider an input
  instance $I$ with an election $E = (C,V)$, where $p$ is the
  preferred candidate, and where the committee size is $k$. Our
  algorithm proceeds as follows.
  
  As first step, we guess the final score that $p$ would have after a
  successful bribery, denoted by $\mathrm{endscore}(p)$.  Since there
  are only polynomially many possibilities, we can simply branch over
  all possible values of $\mathrm{endscore}(p)$ to implement the first
  step.  Then, we consider the set $C' \subseteq C$ of those
  candidates whose scores are greater than $\mathrm{endscore}(p)$.  It
  is clear that to ensure that $p$ is in some winning committee, we
  need to decrease the score of all but $k-1$ candidates from $C'$.
  If $C'$ contains at most $k-1$ candidates, then we do not need to
  decrease the scores of any candidates.

  To this end, we sort the candidates in $C'$ by the cost of
  decreasing their score (by appropriate shifts of $p$) to be equal to
  $\mathrm{endscore}(p)$, and pick all of the candidates in $C'$,
  besides the $k - 1$ most expensive ones. Since for each bribed voter one can
  decrease the score of exactly one candidate, this defines a shift
  action. If this shift action does not guarantee that $p$ has score
  $\mathrm{endscore}(p)$, then we complement it by shifting $p$ to the
  lowest approved position in sufficiently many cheapest votes, to
  ensure that $p$ has score $\mathrm{endscore}(p)$.

  If the thus computed shift action is within budget, then we
  accept. Otherwise, we try another guess of
  $\mathrm{endscore}(p)$. If we try all possibilities without accepting,
  then we reject.
\end{proof}

The situation for \kBorda is different.  
Elkind et al.~\cite{EFS09} have shown that \ShiftBribery is $\np$-hard
for Borda and, so, the same holds for \kBorda.  We show that
Borda-\textsc{Shift Bribery} is $\wone$-hard for parameterization by
the number of voters, resolving a previously open case~\cite{BCFNN14}.\footnoteref{BCFNN14note}
This immediately implies the same hardness results for all our
Borda-based rules.

\begin{theorem}\label{borda_wrt_n}
  Parameterized by the number of voters, \Borda \ShiftBribery is $\w[1]$-hard
  (even for unit price functions).
\end{theorem}

\begin{proof}
  We give a parameterized reduction from 
  the $\w[1]$-hard \probMCIS problem.
  Given a graph $G = (V(G),E(G))$ where each vertex has one of $h$ colors,
  \probMCIS asks whether there are $h$ vertices of pairwise-distinct colors
  such that no two of them are connected by an edge.
  Let~$(G,h)$ be our input instance.
  Without loss of generality, we assume that the
  number of vertices of each color is the same and that there are no
  edges between vertices of the same color.
  We write $V(G)$ to denote the set of $G$'s vertices, and $E(G)$ to
  denote the set of $G$'s edges. Further, for every color $i \in [h]$,
  we write $V^{(i)} = \{v_1^{(i)}, \ldots, v_q^{(i)}\}$ to denote the
  set of vertices of color~$i$. For each vertex $v$, we write $E(v)$
  to denote the set of edges incident to~$v$.  For each vertex $v$, we
  write $\delta(v)$ to denote its degree, i.e., $\delta(v) = |E(v)|$
  and we let $\Delta = \max_{u \in V(G)}\delta(u)$ be the highest
  degree of a vertex~$G$. 

  We form an instance of Borda-\textsc{Shift-Bribery} as follows.  We
  let the candidate set be:
  \[
    C = \{p\} \cup V(G) \cup E(G) \cup F(G) \cup D' \cup D''\text{,}
  \]
  where $F(G)$, $D'$, and $D''$ are sets of special dummy candidates.
  For each vertex $v$, we let $F(v)$ be a set of $\Delta - \delta(v)$
  dummy candidates, and we let $F(G) = \bigcup_{v \in V(G)}F(v)$. We set
  $F({-i}) := \bigcup_{v \in V^{(i')},i'\neq i}F(v)$.  We will
  specify $D'$ and $D''$ later.
  For each vertex $v$, we define the partial preference order $S(v)$:
  \[ 
    S(v) \colon v \pref E(v) \pref F(v).
  \]
  For each color $i$, we define $R(i)$ to be a partial preference
  order that ranks first all members of $D'$, then all vertex
  candidates of colors other than $i$, then all edge candidates
  corresponding to edges that are not incident to a vertex of
  color~$i$, then all dummy vertices from $F(-i)$, and finally all
  candidates from~$D''$.

  We use unit price functions and we set the budget to be
  $B = h(q + (q-1)\Delta)$. We set~$D'$ and~$D''$ to consist of
  $2B$~dummy candidates each.  We create the following voters:
  \begin{enumerate}
  \item
  For each color $i \in [h]$, we introduce four voters, $x_i, x'_i, y_i$, and $y'_i$.
  Voters $x_i$ and $x'_i$ have the following preference orders:
    \begin{align*}
      x_i \colon & S(v^{(i)}_1) \pref S(v^{(i)}_2) \pref \cdots \pref S(v^{(i)}_q) \pref p \pref R(i),  \\
      x_i' \colon & \revnot{S(v^{(i)}_q)} \pref \revnot{S(v^{(i)}_{q-1})}
      \pref \cdots \pref \revnot{S(v^{(i)}_1)} \pref p \pref R(i).
    \end{align*}
    Voters $y_i$ and $y'_i$ have preference orders that are reverses
    of those of $x_i$ and $x'_i$, respectively, except that candidates
    from~$D''$ are ranked last in their votes as well.

  \item
  We create a voter~$z$ with the preference order:
    \[
    z \colon F(G) \pref V(G) \pref E(G) \pref D' \pref p \pref D'',
    \]
    and a voter $z'$ with the preference order that is obtained from that of~$z$
    by first reversing it, and then shifting each member of $V(G) \cup
    E(G)$ by one position forward, and shifting $p$ by $B$
    positions back.
  \end{enumerate}

  Let $L$ be the score of $p$ prior to executing any shift
  actions. 
  Simple calculations show that each candidate in $V(G) \cup E(G)$ has
  score $L + B + 1$, and each candidate in $F(G) \cup D' \cup D''$ has
  score at most $L+B$ (to see this, it suffices to consider voters $z$ and $z'$
  as the other voters have preference orders that are symmetric with
  respect to all the candidates except for those in $D''$, who are
  always ranked last).

  We show that it is possible to ensure the victory of $p$ in our
  election by a bribery of cost at most~$B$ if and only if there is a
  multicolored independent set for $G$ of size $h$.

  For the ``if'' case, we show that if $G$ has a multicolored independent set, then
  there is a successful shift action of cost $B$ in our election.  Let
  us fix a multicolored independent set for $G$ and, for each color $i
  \in [h]$, let $v^{(i)}_{s_i}$ be the vertex of color $i$ from this
  set. For each pair of voters $x_i, x'_i$, we shift $p$ so that in
  $x_i$ he or she ends up right in front of $v^{(i)}_{s_i+1}$ (or $p$
  does not move if $s_i = q$), and in $x'_i$ he or she ends up right in
  front of $v^{(i)}_{s_i}$. This way, $p$ passes every vertex candidate
  from $V^{(i)}$ and every edge candidate from $\left(\bigcup_{t \in
      [q]}E(v^{(i)}_t) \right) \setminus E(v^{(i)}_{s_i})$.  This
  shift action costs $B/h$ for every pair of voters $x_i,x'_i$, so, in
  total, costs exactly $B$. Further, clearly, it ensures that $p$
  passes every vertex candidate so each of them has score~$L+B$.
  Finally, since we chose vertices from an independent set, every edge
  candidate also has score at most~$L+B$:
  If $p$ does not pass some edge~$e$ between vertices of colors~$i$ and~$j$
  for a pair of voters $x_i, x'_i$, then $p$ certainly passes $e$ in the
  pair of votes $x_j, x'_j$ because $v^i_{s_i}$ and $v^j_{s_j}$ are not adjacent.

  For the ``only if'' case, we show that if there is a successful shift action for our
  instance, then there is a multicolored independent set for $G$.  We
  note that a shift action of cost $B$ gives $p$ score $L+B$. Thus,
  for the shift action to be successful, it has to cause all
  candidates in $V(G) \cup E(G)$ to lose a point. We claim that a
  successful shift bribery has to use exactly $B/h = (q +
  (q-1)\Delta)$ unit shifts for every pair of voters $x_i$, $x'_i$.
  Why is this so? Let us fix some color $i \in [h]$.  Every successful
  shift action has to decrease the score of every vertex candidate
  and $x_i$, $x'_i$ are the only votes where $p$ can pass the vertex
  candidates from $V^{(i)}$ without exceeding the budget.
  If we spend fewer than $B/h$ units of budget on $x_i, x'_i$, then there
  will be some vertex candidate corresponding to a vertex from $V^{(i)}$
  that $p$ did not pass (and, in effect, which does not lose a point),
  and so $p$ will not be a winner.
  Thus we know that a successful shift action spends $B/h$ units of budget
  on every pair of voters $x_i, x'_i$. Further, we can assume that for
  each color~$i$ there is a vertex $v_{s_i}^{(i)} \in V^{(i)}$
  such that in~$x_i$ candidate~$p$ is shifted to be right in front of
  $v^{(i)}_{s_i+1}$ and in $x'_i$ candidate~$p$ is shifted to be right
  in front of $v^{(i)}_{s_i}$. We call such a vertex $v^{(i)}_{s_i}$ 
  \emph{selected}.
  If for a given pair of voters $x_i, x'_i$ neither of the vertices
  from~$V^{(i)}$ were selected, then there would be some vertex
  candidate in~$V^{(i)}$ that $p$ does not pass. 
  If for some pair of voters $x_i, x'_i$ vertex
  $v^{(i)}_{s_i}$ is selected, then in this pair of votes $p$ does not pass
  the edge candidates from $E(v^{(i)}_{s_i})$. However, this means that in
  a successful shift action the selected vertices form an independent
  set of~$G$. If two vertices $v^{(i)}_{s_i}$ and $v^{(j)}_{s_j}$ were
  selected, $i \neq j$, and if there were an edge $e$ connecting them,
  then $p$ would not pass the candidate~$e$ in either of the pairs
  of votes $x_i,x'_i$ or $x_j, x'_j$. Since these are the only votes where
  $p$ can pass~$e$ without exceeding the budget, in this case $e$
  would have $L+B+1$ points, while $p$ would have $L+B$ points and would lose. 
\end{proof}
In effect, we have the following corollary 
(we discuss other Borda-based rules later).

\begin{corollary}\label{cor:kBorda-n}
  Parameterized by the number of voters,
 $k$-Borda-\textsc{Shift Bribery} is $\wone$-hard.
\end{corollary}

\autoref{cor:kBorda-n} shows that the $\fpt$ approximation scheme
from \autoref{prop:fpt-as} can presumably not be replaced by an $\fpt$ algorithm.
By \autoref{prop:all-or-nothing-n}, we also know that
$k$-Borda-\textsc{Shift Bribery} is in $\fpt$ for all-or-nothing
prices and the parameterization by the number of voters.

The next result is, perhaps, even more surprising than \autoref{borda_wrt_n}.
It turns out that $k$-Borda-\textsc{Shift Bribery} is $\wone$-hard 
also for the parameterization by the number of unit shifts,
whereas Borda-\textsc{Shift Bribery} is in $\fpt$ for this parameterization.
To this end, we describe a parameterized reduction from \probClique.

\newcommand{\ttkbordawrts}{Parameterized by the number $s$ of unit shifts,
  \kBorda \ShiftBribery is $\w[1]$-hard.}

\begin{theorem}\label{kborda_wrt_s}
  \ttkbordawrts
\end{theorem}

\begin{proof}
  We provide a parameterized reduction from the 
  $\wone$-complete \probClique problem in which we are
  given a graph $G$ with $V(G) = \{v_1, \ldots, v_n\}$ and $E(G) =
  \{e_1, \ldots, e_m\}$ and we ask whether there is a set of $h$ pairwise
  adjacent vertices in~$G$. 

  Given an instance of the \probClique problem, create an instance
  of \kBorda \ShiftBribery as follows.  Set the budget $B := {h
    \choose 2} \cdot (2 + h^3)$, use unit price functions, and
  set the size of the committee $k := n - h + 1$. The
  candidate set is:
  \[ 
     C = \{p\} \cup V(G) \cup D(G) \cup F,
  \] 
  where the sets $D(G)$ and~$F$ contain dummy candidates
  specified as follows.  For each edge $e$ from the graph, let $D(e)$
  be a set of $h^3$ dummy candidates, and let $H$ be a set of $B$ dummy
  candidates. Set
  $D(G) := \left( \bigcup_{e \in E(G)}D(e) \right) \cup H$. Define $F$
  to contain $B + (h-1)$ dummy candidates.

  We form the set of voters as follows:
  \begin{enumerate}
  \item
  For each edge $e = \{u, v\}$ from $G$ we introduce voter $x_e$
    with preference order:
  \[
  u \pref v \pref D(e) \pref p \pref D(G) \setminus D(e) \pref V(G)
  \setminus \{u,v\} \pref F,
  \]
  and voter $y_e$ whose preference order is the reverse of that 
  of~$x_e$, with candidates from $F$ shifted to the bottom positions.

  \item
  We introduce two voters, $z$ and $z'$, where $z$ has preference
  order $V(G) \pref F \pref p \pref D(G)$ and $z'$ has preference
  order $F \pref p  \pref \revnot{V(G)} \pref D(G)$.
  \end{enumerate}

  All vertex candidates have the same score in this election, and we
  denote it by $L$. Candidate $p$ has score $L - (h-1) -B$, and all
  remaining candidates have score lower than $L$ (note that we can
  assume that $G$ has more than $\binom{h}{2}$ edges as otherwise it
  certainly does not contain a size-$h$ clique).
  Intuitively, shifting $p$ to the top positions in votes $x_e$
  corresponding to a size-$h$ clique is the only way to ensure $p$'s~victory

  It remains to be shown the correctness of the construction.
  More precisely, we show that $G$ contains a clique of size $h$ if and only if there is a
  successful shift action for our instance of $k$-Borda-\textsc{Shift Bribery}.

  For the ``if'' case, assume that there is a clique if size~$h$ in~$G$.
  Then, a successful bribery can shift~$p$ to the
  front of all $x_e$ voters corresponding to the edges inside this
  clique. This gives $p$ additional $B$~points and causes each vertex
  from the clique to lose $h-1$ points. In effect, there are $n-h$
  vertex candidates with score higher than that of $p$, and $h$ vertex
  candidates with the same score as $p$. Since all other
  candidates already had lower scores, $p$~belongs to at least one
  winning committee.

  For the ``only if'' case, note that $p$ can join some winning
  committee only if at least $h$ vertex candidates lose $h-1$ points
  each. Without exceeding the budget, $p$ can pass vertex candidates
  only in $x_e$ votes. Through simple arithmetic, we see that within a
  given budget we can shift $p$ to pass some vertex candidates in at
  most $\binom{h}{2}$ of these votes and, so, in each of them we can
  shift $p$ to the top position. That is, a successful shift action
  passes vertices corresponding to $\binom{h}{2}$ edges. This can lead
  to $h$ vertex candidates losing at least $h-1$ points each (or, in fact,
  exactly $h-1$ points each) only if these edges form a size-$h$
  clique.
\end{proof}

\section{Chamberlin-Courant and Its Variants}
\label{sec:CC-results}

We now move on to the Chamberlin-Courant (CC) rules and their approximate variants.
These rules try to find a committee such that every voter is represented well by
some member of the committee.
Recall that \textsc{Winner determination} for Borda-CC and Approval-CC
is NP-hard but can be solved efficiently for the approximate variants.
To some extent, this difference in the computational complexity is
also reflected by our findings for \textsc{Shift Bribery}.

Note that many results for the CC-based rules (see also \autoref{table:results}
in \autoref{sec:general_results})
follow from our results from previous sections.
For the parameterizations by the number of candidates, 
\autoref{thm:m} gives $\fpt$ results for all CC-based rules.
For the parameterization by the number of voters,
by \autoref{prop:greedy-approval-cc-n} we
have $\fpt$ results for Approval-CC, Greedy-Approval-CC, 
and PTAS-CC.
We inherit $\wone$-hardness for Borda-CC and Greedy-Borda-CC 
from \autoref{borda_wrt_n}, since both rules coincide with
the single-winner Borda rule in the case of committee size~$k=1$.

\begin{corollary}\label{cor:BordaCC-n}
  \textsc{Shift Bribery} parameterized by the number of voters is
  $\wone$-hard for Borda-CC and for Greedy-Borda-CC even for unit
  price functions.
\end{corollary}

By \autoref{prop:fpt-as}, we have that there is an
$\fpt$ approximation scheme for Borda-CC. However, since
\autoref{prop:fpt-as} strongly relies on candidate monotonicity of
the rule, it does not apply to Greedy-Borda-CC. Indeed, we conjecture
that there is no constant-factor $\fpt$ approximation algorithm for
Greedy-Borda-CC-\textsc{Shift Bribery} (parameterized by the number of
voters). 
So far we could prove this only for the case of
weighted elections, i.e., for the case where each voter $v$ has an
integer weight $w_v$ and counts as $w_v$ separate voters for computing
the result of the election (but not for the computation of the
parameter). On the one hand, one could say that 
using weighted votes goes against the spirit of
parameterization by the number of voters and, to some extent, we agree.
On the other hand, however, all our $\fpt$ results for
parameterization by the number of voters (including the $\fpt$
approximation scheme in \autoref{prop:fpt-as}) do hold for the weighted case.
By a parameterized reduction from the \textsc{Multicolored Clique}
problem, we obtain the following.

\newcommand{\tthmgbcci}{Unless $\wone=\fpt$, Greedy-Borda-CC-\textsc{Shift Bribery} with
  weighted votes is not $\alpha$-approximable for any constant
  $\alpha$, even in $\fpt$ time with respect to the number of voters
  and even for unit price functions.}

\begin{theorem}\label{thm:greedy-borda-cc-inapprox}
  \tthmgbcci
\end{theorem}

\begin{proof}
  We first prove $\wone$-hardness of the problem and then argue that
  this proof implies the claimed inapproximability result.  

  We give a reduction from the \textsc{Multicolored Clique} problem
  for the case of regular graphs, which is $\wone$-complete for the
  parameter solution size~$h$ (see, e.g., the work of Mathieson and
  Szeider~\cite[Lemma 3.2]{MS12}).  To this end, let $G = (V(G),E(G))$
  be our input graph and let $h$ be the size of the desired clique
  (and the number of vertex colors). We use the following
  notation. For each color $i \in [h]$, we let
  $V^{(i)} = \{v^{(i)}_1, \ldots, v^{(i)}_n\}$ be the set of vertices
  from $G$ with color $i$. For each vertex $v \in V(G)$, we write
  $E(v)$ to denote the set of edges incident to $v$.  Since $G$ is
  regular, we let~$d$ be the common degree of all the vertices (i.e.,
  for each vertex $v$, $|E(v)| = d$).  For each pair of distinct
  colors $i,j \in [h]$, $i < j$, we write $E(i,j)$ to denote the set
  of edges between vertices of color~$i$ and vertices of color~$j$.

  We make the following observation regarding Greedy-Borda-CC.  In
  each iteration it picks a candidate with the highest score, where
  this score is computed as follows: Let $W$ be the set of candidates
  already selected by Greedy-Borda-CC at this point. Consider
  candidate $c$ and voter $v$, and let $d$ be the candidate from $W$
  that $v$ ranks highest. Voter $v$ gives $\max(0,
  \pos_v(d)-\pos_v(c))$ points to $c$ (i.e., the number of points by
  which adding $c$ to $W$ would increase the score of $v$'s
  representative). The score of a candidate in a given iteration is
  the sum of the scores it receives from all voters.
  We form an instance of Greedy-Borda-CC-\textsc{Shift Bribery} as follows.

  \smallskip
  \emph{The candidates.}\quad
  We let the candidate set be $C = \{b,p,p'\} \cup V(G) \cup E(G)
  \cup D$, where $p$ is the preferred candidate,
  $p'$ is $p$'s direct competitor in the sense that either~$p$ or~$p'$
  will be in the committee,
  $b$ is the ``bar'' candidate (see explanation below),
  and $D$ is a set of dummy candidates.
  Throughout the construction we will introduce many dummy candidates
  and we do not give them special names; at the end of the construction
  it will be clear that we add only polynomially many of them.
  We will ensure that $b$, the bar
  candidate, is always chosen first into the committee, so---in
  essence---the scores of all other candidates can be computed
  relative to $b$. Thus, when we describe a preference order, we list
  only top parts of the voters' preference orders, until candidate~$b$.
  Candidate~$p$ is ranked last in every vote in which we do not
  explicitly require something else.

  We also use the following notation in the descriptions of the
  preference orders. For a number $L$, by writing $[L]$ in a
  preference order we mean introducing $L$ new dummy candidates that
  are put in the following positions in this preference order, but
  that in every other preference order are ranked below $b$ (and,
  thus, after $b$ is selected receive no points from these voters).

  \smallskip
  \emph{The voters.}\quad
  We introduce the following voters, where $N$, $T_v$, $T_e$, and $T_p$ are
  four large numbers such that $N$ is much bigger than $T_v$, $T_v$
  is much bigger than $T_e$, and $T_e$ is much bigger than $T_p$; we
  will provide their exact values later.  Each voter has weight one
  unless specified otherwise.
  \begin{enumerate}

  \item For each color $i \in [h]$, we introduce two \emph{vertex-score}
    voters with the following preference orders:
    \begin{align*}
      V^{(i)} \pref [N\cdot(T_v-i)] \pref b,\\
      \revnot{V^{(i)}} \pref [N\cdot(T_v-i)] \pref b,
    \end{align*}
    and two \emph{vertex-selection voters} with the following preference orders:
    \begin{align*}
      V^{(i)} \pref p \pref b, \\
      \revnot{V^{(i)}} \pref p \pref b.
    \end{align*}

  \item For each pair of distinct colors $i, j \in [h]$, $i < j$, we
    introduce two \emph{edge-score voters} with the following preference
    orders:
    \begin{align*}
      E(i,j) \pref [N\cdot(T_e-(i \cdot h + j))] \pref b,\\
      \revnot{E(i,j)} \pref [N\cdot(T_e-(i \cdot h + j))] \pref b,
    \end{align*}
    and two \emph{edge-selection voters} with the following preference orders:
    \begin{align*}
      E(i,j) \pref p \pref b, \\
      \revnot{E(i,j)} \pref p \pref b.
    \end{align*}
    Each of the edge-selection voters has weight $\omega = 8n(d+1)$ 
    (and these are the only voters with non-unit weights). 
    
  \item For each color $i \in [h]$ we introduce two \emph{verification voters}
  with the following preference orders:
    \begin{align*}
      p \pref v^{(i)}_1 \pref E(v^{(i)}_1) \pref \cdots  \pref v^{(i)}_n \pref E(v^{(i)}_n)   \pref b, \\
      p \pref v^{(i)}_n \pref \revnot{E(v^{(i)}_n)} \pref \cdots  \pref v^{(i)}_1 \pref \revnot{E(v^{(i)}_1)}   \pref b. 
    \end{align*}

  \item We introduce the following two voters, the \emph{$p/p'$-score
    voters}, with the following preference orders:
    \begin{align*}
      &p' \pref [N \cdot T_p - h(n+1)(d+1)] \pref b, \\
      &p  \pref [N \cdot T_p - 2h(nd + n + 1)] \pref b.
    \end{align*}

  \item Let $H$ be the total weight of voters introduced so far (clearly,
    $H$ is polynomially upper-bounded in the input size of the \textsc{Multicolored
    Clique} instance $(G,h)$). We introduce $2H+1$ pairs of
    voters with preference orders $b \pref C \setminus \{b\}$ and $b
    \pref \revnot{C \setminus \{b\}}$. We refer to these voters as the
    \emph{bar-score voters}.
  \end{enumerate}

  We assume that the internal tie-breaking prefers $p$ to $p'$---we
  could modify the construction slightly if it were the other way round.

  \smallskip
  \emph{Committee size and budget.}\quad
  We set the committee size to be $k = 1 +h + \binom{h}{2}+1$.
  We use unit prices for the voters and we set the budget $B = |V|-h + |E|-\binom{h}{2}$.

  We claim that for an appropriate choice of $N$, $T_v$, $T_e$, and
  $T_p$ it is possible to ensure that $p$ is in a winning committee
  if and only if there is multicolored size-$h$ clique for $G$.
  We now argue why this is the case.

  \smallskip
  \emph{The idea.}\quad
  The general idea is to show that every shift action (even the zero-vector,
  that means not bribing the voters) of costs at most~$B$ leads to a committee
  that contains:
  \begin{enumerate}
   \item the bar candidate~$b$,
   \item for each color~$i$, one candidate corresponding to a vertex of color~$i$,
   \item for each color pair~$\{i,j\}, i \neq j$ one candidate corresponding to
         an edge incident to a vertex of color~$i$ and to a vertex of color~$j$, and
   \item candidate~$p$ if the selected vertices and edges encode a multicolored
         clique; otherwise the committee contains~$p'$.
  \end{enumerate}
  Furthermore, any such combination of vertices and edges can be selected
  within the given budget, that is, there is a successful shift action
  if a multicolored clique of size~$h$ exists.
  
  \smallskip
  \emph{Correctness.}\quad
  Observe that due to the bar-score voters, irrespective how
  we shift $p$ within the budget, Greedy-Borda-CC will first choose~$b$.
  Thus, from this point on, we compute the score of all
  candidates relative to $b$ (and, in later rounds, the other selected
  members of the committee, but there is a limited number of such
  interactions). 

  We now describe the next $h + \binom{h}{2}+1$ rounds, for each of
  them first describing the situation as if $p$ were not shifted and
  then indicating how the iteration would change with appropriate
  shifts.
  
  After the first iteration, when $b$ is selected, for each color $i \in [h]$,
  every vertex in $V^{(i)}$ has score:
  \[
    \underbrace{(2N\cdot(T_v-i) + (n+1))}_{\text{vertex-score voters}} +
    \underbrace{(n+3)}_{\text{vertex-selection voters}} +
    \underbrace{((n+1)(d+1))}_{\text{verification voters}}.
  \]

  The points in the first bracket come from the vertex-score voters,
  in the second bracket from the vertex-selection voters, and in the
  last bracket from the verification voters. Further, since $T_v$ is
  much larger than $T_e$ and $T_p$, every non-vertex candidate has
  significantly lower score. 

  Thus, in the next $h$ rounds, for each color $i \in [h]$,
  Greedy-Borda-CC adds into the committee one vertex candidate of
  color $i$.  Note that as soon as it picks some vertex candidate of
  color $i$, the scores of all other vertex candidates of this
  color immediately drop by at least $2N\cdot(T_v-i)$ and, so, their
  scores are much too low to be selected in the following rounds.

  By shifting candidate $p$ in the vertex-selection votes, for each
  color $i \in [h]$ and each vertex in~$V^{(i)}$ it is possible to
  ensure that exactly this vertex is selected (it suffices to ensure
  that every other vertex candidate of this color loses one point due
  to $p$ passing him or her). The costs of such shifts are at most
  $|V|-h$ in total. 
  In other words, we can assume that after these $h$ iterations
  Greedy-Borda-CC picks one vertex candidate of each color, and that
  by shift action of cost at most $|V|-h$ it is possible to choose
  precisely which ones.

  In the next $\binom{h}{2}$ iterations, Greedy-Borda-CC picks one
  edge candidate for each pair of colors. Not counting the
  verification voters, for each pair of colors $i,j\in [h]$, $i < j$,
  every edge candidate connecting vertices of colors $i$ and $j$ has
  score:
  \begin{align*}
    \underbrace{(N(T_e -(i \cdot h +j )) + |E(i,j)|+1)}_{\text{edge-score voters}} +
    \underbrace{(\omega (|E(i,j)|+3))}_{\text{edge-selection voters}},
  \end{align*}
  where the points from the first bracket come from the edge-score
  voters and the points in the second bracket come from the
  edge-selection voters. Further, every such candidate receives less than
  $\frac{\omega}{2}$ points from the verification voters.

  Since $T_e$ is much larger than $T_p$, and since by shifting $p$
  forward in the votes of edge-selection voters it is possible to
  remove $\omega$ points from the scores of all but one edge candidate
  in each $E(i,j)$, it is possible to precisely select for each $E(i,j)$ which
  of its members is added to the committee with a shift action of
  total cost $|E|-\binom{h}{2}$.
  Analogously to the case of vertices, whenever some candidate
  from $E(i,j)$ is selected, the other ones lose so many points that they
  have no chance of being selected in any of the following iterations.

  In the final iteration, the algorithm either selects $p'$ or $p$.
  Candidate $p'$ has score $N \cdot T_p - h(n+1)(d+1)$, whereas the
  score of $p$ depends on the vertex and the edge candidates that were so
  far introduced into the committee. If we disregarded all committee
  members selected after $b$, then $p$ would have score $N \cdot T_p$,
  because $p$ receives $N \cdot T_p - 2h(nd + n + 1)$ points from the
  $p/p'$-score voters and $2(nd+n+1)$ points for every color
  $i \in [h]$ from the verification voters.
  However, if we take into account the candidates selected in the preceding
  rounds, then,
  for each color $i \in [h]$, $p$ loses $(n+1)(d+1)$ points
  from the verification voters.  This is true since whenever some candidate
  from $V^{(i)}$ is in the committee, we compute $p$'s score relative
  to this vertex candidate and not relative to $b$.  If these were the
  only points that $p$ lost due to the committee members already
  selected, then---by tie-breaking---$p$ would win against~$p'$.
  However, if for some pair of colors $i,j \in [h]$, $i < j$, the
  committee contained some edge $e$ that connected vertices that are not
  both in the committee, then $p$ would lose at least one more point from
  the verification voters (either for color~$i$ or for color~$j$ or
  for both) because at least one of these verification voters would rank
  $e$ ahead of all the vertex candidates from the committee.  Then
  $p'$ would be selected.  This means that $p$ ends up in the committee if
  and only if due to an appropriate shift action we select vertices
  and edges corresponding to a multicolored clique.  This proves the
  correctness of the reduction for an appropriate choice of $N$,
  $T_v$, $T_e$, and $T_p$, which is discussed next.

  \smallskip
  \emph{The values of $N$, $T_v$, $T_e$, and $T_p$.}\quad
  While one could pick tight precise values, for the
  correctness of the proof it suffices to take, say, $T_p =
  (\binom{h}{2}\cdot|V|\cdot|E|)^3$, $T_e = T_p^3$, $T_v = T_e^3$, and
  $N = T_v^3$.

  Finally, we discuss the inapproximability result that is implied by our reduction.

  \smallskip
  \emph{Inapproximability.}\quad
  Observe that, in fact, the above proof gives our
  inapproximability result. The reason is that for a given constant
  factor $\alpha$, we could increase $N$ by the same factor and it
  would be impossible for $p$ to pass the bar candidate in any of the
  votes, even if we were to spend $\alpha$ times the necessary
  budget. In effect, for $p$ to succeed we would still have to find a
  multicolored clique.
\end{proof}

For the parameterization by the number of unit shifts,
both Borda-CC and Approval-CC are para-$\np$-hard due to the hardness
of \textsc{Winner Determination}.\footnote{The
  literature~\cite{BL11,PRZ08} speaks of hardness of computing the
  score of a winning committee, but one can show that deciding whether
  a given candidate is in some winning committee is $\np$-hard as well
  (and, indeed, this was formally shown by Bredereck et al.~\cite{bre-fal-kac-nie-sko-tal:c:robustness}).}
For Greedy-Approval-CC, PTAS-CC, and Greedy-Borda-CC we obtain
$\wtwo$-hardness results and inapproximability results.
\begin{comment}
This, once again, shows a big difference between the single-winner
case and the multiwinner case.
\end{comment}

\newcommand{\thmunitshiftsga}{Parameterized by the total number~$s$ of
  unit shifts, \ShiftBribery is $\wtwo$-hard even in case of unit prices
  for Greedy-Borda-CC, Greedy-Approval-CC, and PTAS-CC.
  Further, unless $\wtwo=\fpt$, in these cases the problem is not
  $\alpha$-approximable for any constant~$\alpha$.}

\begin{theorem} \label{W2h-unitshifts} 
  \thmunitshiftsga
\end{theorem}

\begin{proof}
  First, we show the result for Greedy-Approval-CC for $t$-Approval
  satisfaction function with $t\ge3$ (which implies the same result for PTAS-CC).
  Second, we show how the proof idea can be adapted to obtain the
  same result for Greedy-Borda-CC.\bigskip
 
  \noindent\emph{Greedy-Approval-CC.}\quad
  We reduce from the \probSC problem, which is \wtwo-hard
  parameterized by the set cover size~$h$.
  Given a set~$U$ of elements, a family~$\calS$
  of subsets of~$U$, and an integer~$h$, \probSC asks whether
  there is a subset of $h$ sets from $\calS$ whose union is $U$.
  Let $(\calS,U,h)$ be an instance of \probSC, where~$\calS = (S_1, \ldots, S_s)$ is a
  collection of sets, $U =\{u_1, \ldots, u_r\}$ is the universe,
  and~$h$ is the solution size.  We construct a \GreedyApprovalCC
  \ShiftBribery instance as follows. \smallskip

  \emph{Important candidates.}\quad
  For each element~$u\in U$, we introduce two \emph{element
    candidates}, $c^-(u)$ and~$c^+(u)$.  Analogously, for each set
  $S \in \calS$, we introduce two \emph{set candidates}, $c^-(S)$
  and~$c^+(S)$.  Furthermore, we introduce the preferred candidate~$p$
  and a candidate~$p'$.\smallskip

  \emph{Dummy candidates.}\quad For each voter (to be specified
  later), we introduce up to $(t-1)$ dummy candidates.  Each dummy
  candidate is approved by exactly one voter, for whom he or she is
  introduced. All the important candidates will have much higher
  scores than the dummy ones and, so, no dummy candidate will join the
  winning committee, irrespective of the shifts of the preferred
  candidate.
  The reason for introducing the dummy candidates is to ensure that
  even though we use $t$-Approval scores, we can construct each voter
  so that he or she approves any number $t'$, $1 \leq t' \leq t$, of
  important candidates (and the remaining $t-t'$ top positions are
  filled with the dummy candidates).\smallskip

  \emph{Committee size and budget.}\quad We set the budget equal to
  the size~$h$ of the set cover and the committee size
  to~$|\calS|+|U|+1$.\smallskip

  \emph{The idea.}\quad The idea of the reduction is to construct an
  election where the Greedy-Approval-CC rule first simulates the
  process of choosing the sets from the \textsc{Set Cover} instance
  and then chooses the elements from the universe that are covered.
 
  We will form the preference orders of the voters so that, prior to
  shifting the preferred candidate, in the first $s$ iterations
  Greedy-Approval-CC would choose candidates
  $c^-(S_1), \ldots, c^-(S_s)$, in the next $r$ iterations it would
  choose candidate $c^-(u_1), \ldots, c^-(u_r)$, and finally it would
  choose~$p'$. However, for each set $S_i \in \calS$, we can shift $p$
  by one position in one vote so that instead of selecting $c^-(S_i)$,
  Greedy-Approval-CC will choose $c^+(S_i)$ in the appropriate round
  (we call such set~$S_i$ selected). Then, for each element~$u$, if
  there is a selected set that contains~$u$, then Greedy-Approval-CC
  will select~$c^+(u)$ instead of~$c^-(u)$ in the appropriate
  round. In the end, if we select $h$~sets that cover all elements
  from the universe, Greedy-Approval-CC will choose~$p$ instead of~$p'$.
  Intuitively, candidates $c^{-}(S)$ correspond to sets $S$ that are
  not included in the \textsc{Set Cover} solution, candidates $c^+(S)$
  correspond to the sets that are used in the solution, candidates
  $c^-(u)$ correspond to elements that are not covered, and candidates
  $c^+(u)$ correspond to the covered ones.

  \smallskip

  \emph{Specifying the voters.}\quad To specify the preference order
  of a voter, for our purposes it suffices to provide the set of at most $t$ important
  candidates that this voter approves (i.e., ranks on the top $t$
  positions) and indicate if one of these candidates is ranked on the
  $t$-th position, right before $p$ (so that $p$ can push this
  candidate out of the approved area and enter it him- or herself). All
  remaining top~$t$ positions are filled with dummy voters. If $p$
  is not ranked on the $t+1$-st position, then we put $p$ on the last
  position in the preference order (so it is impossible to shift $p$
  to an approved position within the budget). All candidates that
  have not been mentioned so far are ranked in some arbitrary order,
  below the top~$t$ positions (or below the top $t+1$ positions, if
  $p$ is on the $t+1$-st one).
  \smallskip

  \emph{The voters.}\quad The set of voters contains $|\calS|$~many
  \emph{$\calS$-voters}, $|\calS|\cdot|U|$~many
  \emph{$\calS$-$U$-voters}, and $|U|$~many $U$-voters:
  \begin{enumerate}
  \item For each set~$S \in \calS$, there is one $S$-voter that
    approves~$c^-(S)$ (and some dummy candidates), so that with a
    single unit shift the preferred candidate $p$ can push $c^-(S)$
    from the approved area and take its place.
  \item For each set~$S \in \calS$ and each element~$u \in U$, there
    is one $S$-$u$-voter that approves: (a) $c^+(S)$, (b) $c^-(u)$
    provided that $u \in S$, and (c) some dummy candidates.
   
  \item For each element~$u \in U$, there is one $u$-voter that
    approves~$p'$ and $c^+(u)$.
  \end{enumerate}
  There are further auxiliary voters that allow us to appropriately
  set the number of approvals for each candidate:
  \begin{enumerate}
  \item For each $S \in \calS$, there are $|\calS|^5\cdot|U|^5-j$
    voters that approve~$c^-(S)$ and~$c^+(S)$ (and some dummy
    candidates).
  \item For each $S \in \calS$, there are additional $|U|-1$~voters
    that approve~$c^-(S)$ (and some dummy candidates).
  \item For each $u \in U$, there are $|\calS|^4\cdot|U|^4-i$ voters
    that approve~$c^-(u)$ and~$c^+(u)$ (and some dummy candidates).

  \item For each $u \in U$, there are additional
    $|\{S \in \calS \mid u \in S\}|-1$ voters that only
    approve~$c^+(u)$ (and some dummy
    candidates). 
  \item There are $|\calS|^2\cdot|U|^2$ voters that approve~$p$
    and~$p'$ (and some dummy candidates).
  \item There are $h-1$~additional voters that only approve~$p'$ (and
    some dummy candidates).
  \end{enumerate}
  By our convention, all but the set voters rank $p$ on the last
  position and, thus, it is too expensive to bribe them to shift $p$
  to an approved position.

  The construction can be computed in polynomial time. Our parameter,
  the number of unit shifts, is upper-bounded by the budget, which is
  identical to the set cover size~$h$.  Before we prove the
  correctness of the reduction, let us briefly discuss the properties
  of the election prior to shifting $p$. \smallskip

  \emph{Scores, ties, and the unbribed election.}\quad First, consider
  the scores of the candidates in the very first round of the voting
  rule, listed below:
  \begin{enumerate}
  \item For each set $S \in \calS$, both candidate~$c^-(S)$ and
    candidate~$c^+(S)$ have $|\calS|^5\cdot|U|^5-j + |U|$ approvals.
  \item For each element $u \in U$, both candidate~$c^-(u_i)$ and
    candidate~$c^+(u_i)$ have
    $|\calS|^4\cdot|U|^4-i + |\{S \in \calS \mid u_i \in S\}|$
    approvals.
  \item Candidate~$p'$ has $|\calS|^2\cdot|U|^2 + |U|+ h-1$ approvals.
  \item Candidate~$p$ has $|\calS|^2\cdot|U|^2$ approvals.
  \end{enumerate}
  We set the tie-breaking order of \GreedyApprovalCC so that
  candidate~$p'$ is preferred to candidate~$p$ and for each
  $x \in U \cup\calS$, candidate~$c^-(x)$ is preferred to
  candidate~$c^+(x)$.\footnote{The reduction can be adapted to work
    for any given tie-breaking.}

  One can verify that in the unbribed election the candidates will
  join the committee in the following order:
  $c^-(S_1),c^-(S_2),\dots,c^-(S_s),c^-(u_1),c^-(u_2),\dots,c^-(u_r)$,
  and finally~$p'$.  To see this, note that for each
  $x \in U \cup \calS$, each pair of candidates~$c^-(x)$ and~$c^+(x)$
  is approved by almost the same sets of voters.  As soon as one
  of~$c^-(x)$ and~$c^+(x)$ joins the committee, the other loses nearly
  all approvals and has no chance of joining the
  committee. Furthermore, the candidates corresponding to the sets
  have higher numbers of approvals than those corresponding to the
  elements, and within both groups the numbers of approvals decrease
  as the indices of the respective sets and elements
  increase. Finally, tie-breaking ensures that the \GreedyApprovalCC
  chooses the $c^-$ candidates.\smallskip

  \emph{Candidate scores.}\quad In the following text, we will often
  speak of the scores of the candidates. For a given round (always
  clear from the context), the score of a candidate is the number of
  voters that approve this candidate and do not approve any candidate
  already included in the committee in the preceding rounds.
  \smallskip

  \emph{The impact of shifting $p$.}\quad The only shift actions that
  affect the result of the election and that are within the given
  budget regard up to $B = h$ set voters.  Let
  $c^-(S_{j_1}),\dots,c^-(S_{j_h})$ be the candidates that were
  originally approved by the bribed set voters instead of
  approving~$p$.  We call $\calS^*:=\{S_{j_1},S_{j_2},\dots,S_{j_h}\}$
  the \emph{selected sets}.  Applying the corresponding shift actions
  will decrease the score of each candidate
  $c^-(S_{j_\ell}), 1 \le \ell \le h$, by one and increase the score
  of~$p$ by~$h$.  One can verify that, in effect, for each
  $1 \le \ell \le h$, \GreedyApprovalCC will select $c^+(S_{j_\ell})$
  instead of $c^-(S_{j_\ell})$ to join the committee in the respective
  round.  Now, observe that for each $u \in S_{j_\ell}$, there is one
  voter that approves $c^+(S_{j_\ell})$ and $c^-(u)$.  This means that
  the score of each~$c^-(u)$ for $u \in \bigcup_{S \in \calS^*}S$ is
  decreased by at least one after the first~$|\calS|$ candidates
  joined the committee.  Hence, if $c^+(S_{j_\ell})$ joins the
  committee instead of $c^-(S_{j_\ell})$, then also~$c^+(u)$ joins
  instead of~$c^-(u)$.  Finally, observe that, after $s+r$ candidates
  joined the committee, the score of~$p'$ is decreased by the number
  of candidates~$c^+(u)$ that joined the committee instead of $c^-(u)$
  (this is due to the $U$-voters).\smallskip 

  \emph{Correctness.}\quad We show that there is a subset of $h$~sets
  from~$\calS$ whose union is~$U$ if and only if there is a successful
  set of shift actions of cost~$h$.
  For the ``if'' case, assume that there is a set
  $\calS' \subseteq \calS$ of $h$~sets whose union is~$U$.  Then,
  bribing the $S$-voter for each~$S\in \calS'$ to approve~$p$ instead
  of~$c^-(S)$ costs~$h$ and successfully makes~$p$ a winner: From the
  above discussion about the impact of shift actions, we immediately
  conclude that after the first $s+r$ rounds, the committee contains
  all the $c^+(u)$ candidates, so the score of $p'$ is
  $|\calS|^2\cdot|U|^2+h-1$, and the score of $p$ is
  $|\calS|^2\cdot|U|^2+h$. Thus, in the final round, $p$ is included
  in the committee instead of $p'$.

  For the ``only if'' case, assume that there is a shift action with
  cost~$h$ that makes $p$~join the committee.  Since $p$~can gain at
  most~$h$ points, $p'$ has to lose at least~$|U|$ points (for the
  final round).  However, the only (important) candidates that are
  approved together with~$p'$ by some voters are the element
  candidates~$c^+(u)$.  To decrease the score of~$p'$ by~$|U|$, all
  the $c^+(u)$ candidates must join the committee instead of the
  $c^-(u)$ candidates.  From the above discussion about the impact of
  shift actions, this is possible only if the union of the selected
  sets is~$U$.\smallskip 

  \emph{Inapproximability.}\quad With a slight modification of the
  above construction, we obtain (fixed-parameter)
  inapproximability. Let $\alpha > 1$ be the considered approximation
  ratio.  First, note that even within a budget of~$\alpha \cdot B$
  one can only afford to bribe the set voters, because in all
  other voters $p$~is ranked last (if there are fewer than
  $\alpha \cdot B$ candidates between the last position in a vote and
  the $t$-th one, then we add sufficiently many never-approved dummy
  candidates).  Second, introduce another pair of important
  candidates, $d$ and~$d'$, and let the set voters additionally
  approve~$d$.  Next, introduce:
  \begin{enumerate}
  \item $|\calS|^3\cdot|U|^3$ voters that approve~$d$ and~$d'$ (and
    some dummy candidates),
  \item $|S|-h$~voters that only approve~$d'$ (and some dummy
    candidates),
  \item $|\calS|\cdot|U|$ voters that approve~$p'$ and~$d$ (and some
    dummy candidates), and further 
  \item $|\calS|\cdot|U|$ voters that only approve~$d'$ (and some
    dummy candidates).
  \end{enumerate}
  Finally, set the tie-breaking so that $d$ is preferred to $d'$ and
  increase the committee size by one.

  The first $|\calS|+|U|$ rounds of the \GreedyApprovalCC procedure
  proceed as in the original construction.  As long as at most~$h$ set
  voters are bribed, candidate~$d$ will join the committee in
  round~$|\calS|+|U|+1$.  In consequence, candidate~$d'$ loses almost
  all points and has no chance to join the committee, and
  candidate~$p'$ loses all additional approvals (introduced by the
  $|\calS|\cdot|U|$ new voters that approve both $p'$ and $d$).  That
  is, the last round proceeds as in the original construction.
  However, if one bribes more than~$h$ set voters, then candidate~$d'$
  will join the committee in round~$|\calS|+|U|+1$, $p'$~keeps the
  additionally introduced approvals, and $p$~has no chance to join the
  committee in the last round.

  It follows that, even with a budget of~$\alpha\cdot B$, one can only
  make~$p$ become member of a winning committee if one selects a
  subset of at most~$h$ sets from~$\calS$ whose union is~$U$.\bigskip

  \noindent
  \emph{Greedy-Borda-CC.}\quad For the case of Greedy-Borda-CC we also
  give a reduction from the \probSC problem.  The basic idea of the
  construction is very similar to that in the proof for
  Greedy-Approval-CC.  However, to implement this idea, we also use
  some concepts from the proof of
  \autoref{thm:greedy-borda-cc-inapprox}.  To this and, we use the
  same notational conventions as in the proof of
  \autoref{thm:greedy-borda-cc-inapprox}, and we use the bar candidate
  in the same way.

  Given an instance $(\calS,U,h)$ of \probSC
  with $\calS = (S_1, \ldots, S_s)$ denoting the given sets over
  the universe~$U =\{u_1, \ldots, u_r\}$,
  we construct a \GreedyBordaCC \ShiftBribery instance as follows.

  We form the following set of candidates:
  \begin{enumerate} 
  \item We introduce the preferred candidate $p$, his or her opponent
    $p'$, and the bar candidate $b$.
  \item For each set $S_i \in \calS$, we introduce two candidates $c^-(S_i)$ and $c^+(S_i)$.
  \item For each element $u_j \in U$, we introduce candidates $c^-(u_j)$
    and $c^+(u_j)$.
  \item We introduce sufficiently many dummy candidates.
  \end{enumerate}

  Let $N$, $T_s$, $T_u$, and $T_p$ be some sufficiently large numbers
  such that $N$ is much larger than $T_s$, $T_s$ is much larger than
  $T_u$, and $T_u$ is much larger than $T_p$ (we will specify their
  values later).  We introduce the following voters:
  \begin{enumerate}
  \item For each set $S_i \in \calS$, we introduce two \emph{set-score
    voters} with preference orders:
    \begin{align*}
      c^-(S_i) \pref c^+(S_i) \pref [N\cdot(T_s-i)] \pref b, \\
      c^+(S_i) \pref c^-(S_i) \pref [N\cdot(T_s-i)] \pref b.
    \end{align*} Further, for each set we introduce two \emph{set-selection
    voters} with preference orders:
    \begin{align*}
    c^-(S_i) \pref p \pref b,\\
    c^+(S_i) \pref [1] \pref b.
    \end{align*}

  \item For each element $u_j \in U$, we introduce two \emph{element-score
    voters} with preference orders:
    \begin{align*}
       c^-(u_j) \pref c^+(u_j) \pref [N\cdot(T_u-j)] \pref b, \\
       c^+(u_j) \pref c^-(u_j) \pref [N\cdot(T_u-j)] \pref b.
    \end{align*}

  \item For each $u_j \in U$, we introduce a \emph{verification voter} $c^+(u_j)
    \pref p' \pref b$.

  \item For each element $u_j \in U$, and each set $S_i \in \calS$
    such that $u_j \in S_i$, we introduce a \emph{covering voter} with
    preference order:
    \[ c^+(S_i) \pref c^-(u_j) \pref b. \] Further, for each candidate $c \in U
    \cup \{c^-(S_1), c^+(S_1)$, $\ldots, c^-(S_s), c^+(S_s)\}$, we introduce
    exactly so many \emph{filler voters} with preference orders of the form $c \pref b$
    so that, relative to $b$, all these candidates receive the same score from
    the verification, covering, and filler voters (taken together).

  \item We introduce two \emph{$p/p'$-score voters} with preference orders
    $p' \pref [N\cdot(T_p) + 2h] \pref b$ and $p \pref [N\cdot(T_p)] \pref
    b$.

  \item Let $H$ be the number of voters introduced so far (clearly,
    $H$ is polynomially upper-bounded in the size of the input instance). We
    introduce $2H+1$ pairs of voters with preference orders $b \pref C
    \setminus \{b\}$ and $b \pref \revnot{C \setminus \{b\}}$. We
    refer to these voters as the \emph{bar-score voters}.
  \end{enumerate}

  We set the committee size to be $1 + s + r +1$, and we set the
  budget $B = h$. We use unit price functions. The internal
  tie-breaking is such that $p$ precedes $p'$, for each $S_i \in
  \calS$, $c^-(S_i)$ precedes~$c^+(S_i)$, and for each $u_j \in U$,
  $c^-(u_j)$ precedes~$c^+(u_j)$.

  The correctness proof is analogous to that for \GreedyApprovalCC.
  To see this, let us now analyze how \GreedyBordaCC proceeds on the
  just-constructed election.  As in the proof of
  \autoref{thm:greedy-borda-cc-inapprox}, it is clear that in the
  first iteration it picks $b$.  Due to the values of $N$ and $T_s$, in
  the next $s$ iterations, for each $S_i \in \calS$, Greedy-Borda-CC
  either adds $c^-(S_i)$ to the committee or it adds $c^+(S_i)$ to the
  committee. With a shift action of cost $h$---by shifting $p$ forward
  in the votes of the set-selection voters---we can select which $h$
  of the $c^+(S_i)$ candidates are introduced into the committee (indeed,
  we need to introduce~$h$ of them to increase $p$'s score---in the final
  iteration---by~$h$). 

  In the next $r$ iterations, for each $j$ Greedy-Borda-CC picks
  either $c^-(u_j)$ or $c^+(u_j)$. One can verify that it picks
  exactly those $c^+(u_j)$ candidates for which in the preceding
  iterations it has picked at least one candidate $c^+(S_i)$ such that
  $u_j \in S_i$ (due to the covering voters).

  In the final iteration, Greedy-Borda-CC either picks~$p$ or~$p'$. It
  picks the former one exactly if it managed to pick $h$ candidates
  from $\calS':=\{c^+(S_{j_1}), \ldots, c^+(S_{j_h})\}$ and all
  candidates $c^+(u_j)$ (since then $p'$ loses $|U|$ points from the
  verification voters and has score $N\cdot T_p + 1 +2h$, $p$ has 
  the same score, as it gets $2h$ points from the set selection
  voters, and tie-breaking prefers $p$ to $p'$).
  This happens if and only if we applied a shift action that ensured
  selection of those $h$ of the $c^+(S_i)$ candidates that correspond
  to a set cover, that is, $\bigcup_{S\in\calS'}S=U$.

  To complete the proof for the \GreedyBordaCC case, we need to pick the
  values of $N$, $T_s$,  $T_u$, and $T_p$.
  It is easy to see that the values $T_p = (r\cdot s \cdot h)^3$,
  $T_u = T_p^3$, $T_s = T_u^3$, and $N = T_s^3$ suffice.

  This proves $\wtwo$-hardness of \textsc{Shift-Bribery} for
  Greedy-Borda-CC.
  To see the inapproximability result, one can use an extension to
  the construction that works analogously to the extension in the
  proof for \GreedyApprovalCC.
\end{proof}

\section{Conclusion}
\label{sec:conclusion}

We studied the complexity of \textsc{Shift Bribery} for two families
of multiwinner rules:
  one, represented by SNTV, Bloc, and $k$-Borda, in which 
rules pick $k$ best candidates according to appropriate single-winner
scoring rules, and another of Chamberlin-Courant rules and their
approximate variants, which focus on providing good representatives.
While we have shown low complexity for SNTV and Bloc
(just like for the single-winner rules on which they are based), we
have shown that \textsc{Shift Bribery} is significantly harder to
solve for $k$-Borda than for its single-winner variant, Borda. The
situation is even more dramatic for the Chamberlin-Courant family of
rules, where in addition to $\wone$- and $\wtwo$-hardness results, we
also obtain inapproximability results.

We focused on the case where we want to ensure a candidate's
membership in \emph{some} winning committee, but it would also be
natural to require membership in \emph{all} winning committees.  In
fact, all our results hold in this model as well, via simple tweaks
(and, in particular, the results for \GreedyBordaCC,
\GreedyApprovalCC, and PTAS-CC already are in this setting because
these rules always produce a single committee).

\iffalse
 Below we briefly
explain why this is so for the tractability results (for the
intractability ones, it requires minor tweaks).

For the tractability results with respect to the \emph{number of candidates}
this can be seen as follows.
\begin{itemize}
 \item
For \SNTV, \Bloc, and \kBorda, we can ensure in our ILP formulations
that the score of~$p$ is strictly greater than the score of the
candidates which are not part of the committee.
 \item
For the round-based rules, the committee is always unique and, hence,
our results already apply.
 \item
For the CC rules, we can build upon the maximum matching
algorithm of \citet{BSU13} (trying matchings where~$p$ is already
matched to one part of the voters, and other ones, where~$p$ is not
matched at all).
\end{itemize}

For the tractability results with respect to the \emph{number of voters}
or \emph{number of shifts}, our algorithms basically try all bribed
elections where~$p$ is in at least one winning committee (except for
the FPT-AS, where we overshoot; due to monotonicity, this does not
hurt). Then, for each bribed election we can adopt the \textsc{Winner Determination}
algorithm of \citet[Proposition 1]{BSU13} that partitions the voters into groups of
voters with the same representative and checks whether~$p$ is part of all cheapest
matchings of representatives to candidates (basically checking all possible partitions).

\fi

Putting an even more demanding bribery goal of involving more than one
candidate to become part of the winning committee(s) is left for
future studies.  Areas of future research also include studying
bribery problems for multiwinner settings with partial preference
orders and studying multiwinner rules based on the Condorcet criterion
(for the hardness of winner determination in such rules, see the works
of Sekar et al.~\cite{sek-sik-xia:c:gehrlein-stable} and Aziz et
al.~\cite{azi-elk-fal-lac-sko:c:gehrlein-and-local}).
Furthermore, our FPT algorithms with
respect to the parameter number of candidates rely on integer linear
programming formulations.  It seems challenging to replace these
algorithms by direct combinatorial algorithms that give us a better
understanding of the problems and potentially better running times.
This reflects a general challenge in the context of parameterized
algorithms for Computation Social Choice~\cite[Key question
1]{BCFGNW14}.

\section*{Acknowledgments}
The authors were supported in part by the DFG project PAWS (NI
369/10) and the NCN project DEC-2012/06/M/ST1/00358.
Nimrod Talmon, while being at TU~Berlin, was supported by the DFG Research Training Group ``Methods for Discrete Structures'' (GRK 1408).
Piotr Faliszewski's visits to TU Berlin were supported by the COST action IC1205 and
by a Friedrich Wilhelm Bessel Research Award from the Alexander-von-Humboldt Foundation, Bonn, Germany.

{\small
\bibliographystyle{abbrvnat}

\begin{thebibliography}{51}
\providecommand{\natexlab}[1]{#1}
\providecommand{\url}[1]{\texttt{#1}}
\expandafter\ifx\csname urlstyle\endcsname\relax
  \providecommand{\doi}[1]{doi: #1}\else
  \providecommand{\doi}{doi: \begingroup \urlstyle{rm}\Url}\fi

\bibitem[Aziz et~al.(2015{\natexlab{a}})Aziz, Brill, Conitzer, Elkind, Freeman,
  and Walsh]{ABCEFW15}
H.~Aziz, M.~Brill, V.~Conitzer, E.~Elkind, R.~Freeman, and T.~Walsh.
\newblock Justified representation in approval-based committee voting.
\newblock In \emph{Proceedings of the 29th AAAI Conference on Artificial
  Intelligence (AAAI~'15)}, pages 784--790. AAAI Press, 2015{\natexlab{a}}.

\bibitem[Aziz et~al.(2015{\natexlab{b}})Aziz, Gaspers, Gudmundsson, Mackenzie,
  Mattei, and Walsh]{AGGMMW15}
H.~Aziz, S.~Gaspers, J.~Gudmundsson, S.~Mackenzie, N.~Mattei, and T.~Walsh.
\newblock Computational aspects of multi-winner approval voting.
\newblock In \emph{Proceedings of the 14th International Conference on
  Autonomous Agents and Multiagent Systems (AAMAS~'15)}, pages 107--115. {ACM},
  2015{\natexlab{b}}.

\bibitem[Aziz et~al.(2017)Aziz, Elkind, Faliszewski, Lackner, and
  Skowron]{azi-elk-fal-lac-sko:c:gehrlein-and-local}
H.~Aziz, E.~Elkind, P.~Faliszewski, M.~Lackner, and P.~Skowron.
\newblock The {Condorcet} principle for multiwinner elections: From
  shortlisting to proportionality.
\newblock In \emph{Proceedings of the 26th International Joint Conference on
  Artificial Intelligence (IJCAI~'17)}, pages 84--90, 2017.

\bibitem[Baumeister et~al.(2012)Baumeister, Faliszewski, Lang, and
  Rothe]{BFLR12}
D.~Baumeister, P.~Faliszewski, J.~Lang, and J.~Rothe.
\newblock Campaigns for lazy voters: {T}runcated ballots.
\newblock In \emph{Proceedings of the 11th International Conference on
  Autonomous Agents and Multiagent Systems (AAMAS~'12)}, pages 577--584.
  IFAAMAS, 2012.

\bibitem[Betzler et~al.(2013)Betzler, Slinko, and Uhlmann]{BSU13}
N.~Betzler, A.~Slinko, and J.~Uhlmann.
\newblock On the computation of fully proportional representation.
\newblock \emph{Journal of Artificial Intelligence Research}, 47:\penalty0
  475--519, 2013.

\bibitem[Brams and Kilgour(2014)]{BK14}
S.~Brams and M.~Kilgour.
\newblock Satsifaction approval voting.
\newblock In R.~Fara, D.~Leech, and M.~Salles, editors, \emph{Voting Power and
  Procedures: {Essays} in Honour of Dan Felsenthal and Mosh{\'e} Machover},
  pages 323--346. Springer, 2014.

\bibitem[Bredereck et~al.(2014{\natexlab{a}})Bredereck, Chen, Faliszewski, Guo,
  Niedermeier, and Woeginger]{BCFGNW14}
R.~Bredereck, J.~Chen, P.~Faliszewski, J.~Guo, R.~Niedermeier, and G.~J.
  Woeginger.
\newblock Parameterized algorithmics for computational social choice: Nine
  research challenges.
\newblock \emph{Tsinghua Science and Technology}, 19\penalty0 (4):\penalty0
  358--373, 2014{\natexlab{a}}.

\bibitem[Bredereck et~al.(2014{\natexlab{b}})Bredereck, Chen, Faliszewski,
  Nichterlein, and Niedermeier]{BCFNN14}
R.~Bredereck, J.~Chen, P.~Faliszewski, A.~Nichterlein, and R.~Niedermeier.
\newblock Prices matter for the parameterized complexity of shift bribery.
\newblock In \emph{Proceedings of the 28th AAAI Conference on Artificial
  Intelligence (AAAI~'14)}, pages 1398--1404. AAAI Press, 2014{\natexlab{b}}.

\bibitem[Bredereck et~al.(2015)Bredereck, Faliszewski, Niedermeier, Skowron,
  and Talmon]{BFNST15}
R.~Bredereck, P.~Faliszewski, R.~Niedermeier, P.~Skowron, and N.~Talmon.
\newblock Elections with few candidates: {Prices}, weights, and covering
  problems.
\newblock In \emph{Proceedings of the 4th International Conference on
  Algorithmic Decision Theory (ADT~'15)}, volume 9346 of \emph{LNCS}, pages
  414--431, 2015.

\bibitem[Bredereck et~al.(2016{\natexlab{a}})Bredereck, Chen, Nichterlein,
  Faliszewski, and Niedermeier]{BCFNN16}
R.~Bredereck, J.~Chen, A.~Nichterlein, P.~Faliszewski, and R.~Niedermeier.
\newblock Prices matter for the parameterized complexity of shift bribery.
\newblock \emph{Information and Computation}, 251:\penalty0 140--164,
  2016{\natexlab{a}}.

\bibitem[Bredereck et~al.(2016{\natexlab{b}})Bredereck, Faliszewski,
  Niedermeier, and Talmon]{BFNT16}
R.~Bredereck, P.~Faliszewski, R.~Niedermeier, and N.~Talmon.
\newblock Large-scale election campaigns: Combinatorial shift bribery.
\newblock \emph{Journal of Artificial Intelligence Research}, 55:\penalty0
  603--652, 2016{\natexlab{b}}.

\bibitem[Bredereck et~al.(2016{\natexlab{c}})Bredereck, Faliszewski,
  Niedermeier, and Talmon]{preliminary}
R.~Bredereck, P.~Faliszewski, R.~Niedermeier, and N.~Talmon.
\newblock Complexity of shift bribery in committee elections.
\newblock In \emph{Proceedings of the 30th AAAI Conference on Artificial
  Intelligence (AAAI~'16)}, pages 2452--2458, 2016{\natexlab{c}}.

\bibitem[Bredereck et~al.(2017)Bredereck, Faliszewski, Kaczmarczyk,
  Niedermeier, Skowron, and Talmon]{bre-fal-kac-nie-sko-tal:c:robustness}
R.~Bredereck, P.~Faliszewski, A.~Kaczmarczyk, R.~Niedermeier, P.~Skowron, and
  N.~Talmon.
\newblock Robustness among multiwinner voting rules.
\newblock In \emph{Proceedings of the 10th International Symposium on
  Algorithmic Game Theory (SAGT~'17)}, volume 10504 of \emph{LNCS}, pages
  80--92. Springer, 2017.

\bibitem[Caragiannis et~al.(2014)Caragiannis, Kaklamanis, Karanikolas, and
  Procaccia]{CKKP14}
I.~Caragiannis, C.~Kaklamanis, N.~Karanikolas, and A.~Procaccia.
\newblock Socially desirable approximations for {D}odgson's voting rule.
\newblock \emph{ACM Transactions on Algorithms}, 10\penalty0 (2):\penalty0
  Article No. 6, 2014.

\bibitem[Cary(2011)]{Car11}
D.~Cary.
\newblock Estimating the margin of victory for instant-runoff voting.
\newblock Presented at the 2011 Electronic Voting Technology Workshop/Workshop
  on Trustworthy Elections, 2011.

\bibitem[Chamberlin and Courant(1983)]{CC83}
B.~Chamberlin and P.~Courant.
\newblock Representative deliberations and representative decisions:
  {Proportional} representation and the {B}orda rule.
\newblock \emph{American Political Science Review}, 77\penalty0 (3):\penalty0
  718--733, 1983.

\bibitem[Cygan et~al.(2015)Cygan, Fomin, Kowalik, Lokshtanov, Marx, Pilipczuk,
  Pilipczuk, and Saurabh]{CFKLMPPS15}
M.~Cygan, F.~V. Fomin, L.~Kowalik, D.~Lokshtanov, D.~Marx, M.~Pilipczuk,
  M.~Pilipczuk, and S.~Saurabh.
\newblock \emph{Parameterized Algorithms}.
\newblock Springer, 2015.

\bibitem[Dorn and Schlotter(2012)]{DS12}
B.~Dorn and I.~Schlotter.
\newblock Multivariate complexity analysis of swap bribery.
\newblock \emph{Algorithmica}, 64\penalty0 (1):\penalty0 126--151, 2012.

\bibitem[Downey and Fellows(2013)]{DF13}
R.~G. Downey and M.~R. Fellows.
\newblock \emph{Fundamentals of Parameterized Complexity}.
\newblock Springer, 2013.

\bibitem[Elkind and Faliszewski(2010)]{EF10}
E.~Elkind and P.~Faliszewski.
\newblock Approximation algorithms for campaign management.
\newblock In \emph{Proceedings of the 6th International Workshop On Internet
  And Network Economics (WINE~'10)}, volume 6484 of \emph{LNCS}, pages
  473--482. Springer, 2010.

\bibitem[Elkind et~al.(2009)Elkind, Faliszewski, and Slinko]{EFS09}
E.~Elkind, P.~Faliszewski, and A.~Slinko.
\newblock Swap bribery.
\newblock In \emph{Proceedings of the 2nd International Symposium on
  Algorithmic Game Theory (SAGT~'09)}, volume 5814 of \emph{LNCS}, pages
  299--310. Springer, 2009.

\bibitem[Elkind et~al.(2014)Elkind, Faliszewski, Skowron, and Slinko]{EFSS14}
E.~Elkind, P.~Faliszewski, P.~Skowron, and A.~Slinko.
\newblock Properties of multiwinner voting rules.
\newblock In \emph{Proceedings of the 13th International Conference on
  Autonomous Agents and Multiagent Systems (AAMAS~'14)}, pages 53--60. IFAAMAS,
  May 2014.

\bibitem[Elkind et~al.(2017)Elkind, Faliszewski, Skowron, and Slinko]{EFSS17}
E.~Elkind, P.~Faliszewski, P.~Skowron, and A.~Slinko.
\newblock Properties of multiwinner voting rules.
\newblock \emph{Social Choice and Welfare}, 48\penalty0 (3):\penalty0 599--632,
  2017.

\bibitem[Faliszewski and Rothe(2015)]{FR15}
P.~Faliszewski and J.~Rothe.
\newblock Control and bribery in voting.
\newblock In F.~Brandt, V.~Conitzer, U.~Endriss, J.~Lang, and A.~D. Procaccia,
  editors, \emph{Handbook of Computational Social Choice}, chapter~7. Cambridge
  University Press, 2015.

\bibitem[Faliszewski et~al.(2009)Faliszewski, Hemaspaandra, and
  Hemaspaandra]{FHH09}
P.~Faliszewski, E.~Hemaspaandra, and L.~A. Hemaspaandra.
\newblock How hard is bribery in elections?
\newblock \emph{Journal of Artificial Intelligence Research}, 35:\penalty0
  485--532, 2009.

\bibitem[Faliszewski et~al.(2015)Faliszewski, Reisch, Rothe, and
  Schend]{FRRS15}
P.~Faliszewski, Y.~Reisch, J.~Rothe, and L.~Schend.
\newblock Complexity of manipulation, bribery, and campaign management in
  {Bucklin} and fallback voting.
\newblock \emph{Autonomous Agents and Multi-Agent Systems}, 29\penalty0
  (6):\penalty0 1091--1124, 2015.

\bibitem[Faliszewski et~al.(2017{\natexlab{a}})Faliszewski, Skowron, Slinko,
  and Talmon]{fal-sko-sli-tal:b:multiwinner-voting}
P.~Faliszewski, P.~Skowron, A.~Slinko, and N.~Talmon.
\newblock Multiwinner voting: {A} new challenge for social choice theory.
\newblock In U.~Endriss, editor, \emph{Trends in Computational Social Choice}.
  AI Access Foundation, 2017{\natexlab{a}}.

\bibitem[Faliszewski et~al.(2017{\natexlab{b}})Faliszewski, Skowron, and
  Talmon]{FST17}
P.~Faliszewski, P.~Skowron, and N.~Talmon.
\newblock Bribery as a measure of candidate success: {Complexity} results for
  approval-based multiwinner rules.
\newblock In \emph{Proceedings of the 16th International Conference on
  Autonomous Agents and Multiagent Systems (AAMAS~'17)}, pages 6--14,
  2017{\natexlab{b}}.

\bibitem[Flum and Grohe(2006)]{FG06}
J.~Flum and M.~Grohe.
\newblock \emph{Parameterized Complexity Theory}.
\newblock Springer, 2006.

\bibitem[Fredman and Tarjan(1987)]{FT87}
M.~L. Fredman and R.~E. Tarjan.
\newblock Fibonacci heaps and their uses in improved network optimization
  algorithms.
\newblock \emph{Journal of the ACM}, 34\penalty0 (3):\penalty0 596--615, 1987.

\bibitem[Hemmecke et~al.(2013)Hemmecke, Onn, and
  Romanchuk]{hem-onn-rom:j:n-fold-ip}
R.~Hemmecke, S.~Onn, and L.~Romanchuk.
\newblock $n$-fold integer programming in cubic time.
\newblock \emph{Mathematical Programming}, 137\penalty0 (1-2):\penalty0
  325--341, 2013.

\bibitem[Kaczmarczyk and Faliszewski(2016)]{kac-fal:c:dsb}
A.~Kaczmarczyk and P.~Faliszewski.
\newblock Algorithms for destructive shift bribery.
\newblock In \emph{Proceedings of the 15th International Conference on
  Autonomous Agents and Multiagent Systems (AAMAS~'16)}, pages 305--313.
  IFAAMAS, 2016.

\bibitem[Kannan(1987)]{Kan87}
R.~Kannan.
\newblock Minkowski's convex body theorem and integer programming.
\newblock \emph{Mathematics of Operations Research\bibremark{No publisher.}},
  12\penalty0 (3):\penalty0 415--440, 1987.

\bibitem[Knop et~al.(2017)Knop, Kouteck{\'{y}}, and Mnich]{KMM17}
D.~Knop, M.~Kouteck{\'{y}}, and M.~Mnich.
\newblock Voting and bribing in single-exponential time.
\newblock In \emph{Proceedings of the 34th International Symposium on
  Theoretical Aspects of Computer Science (STACS~'17)}, volume~66 of
  \emph{LIPIcs}, pages 46:1--46:14. Schloss Dagstuhl--Leibniz-Zentrum fuer
  Informatik, 2017.

\bibitem[Lackner and Skowron(2018)]{lac-sko:c:approval-thiele}
M.~Lackner and P.~Skowron.
\newblock Consistent approval-based multi-winner rules.
\newblock In \emph{Proceedings of the 19th ACM Conference on Electronic
  Commerce (EC~'18)}, pages 47--48, 2018.

\bibitem[Lenstra(1983)]{Len83}
H.~W. Lenstra.
\newblock Integer programming with a fixed number of variables.
\newblock \emph{\bibremark{No string.}Mathematics of Operations
  Research\bibremark{No publisher.}}, 8\penalty0 (4):\penalty0 538--548, 1983.

\bibitem[Lu and Boutilier(2011)]{BL11}
T.~Lu and C.~Boutilier.
\newblock Budgeted social choice: From consensus to personalized decision
  making.
\newblock In \emph{Proceedings of the 22nd International Joint Conference on
  Artificial Intelligence (IJCAI~'11)}, pages 280--286. AAAI Press, 2011.

\bibitem[Magrino et~al.(2011)Magrino, Rivest, Shen, and Wagner]{MRSW11}
T.~Magrino, R.~Rivest, E.~Shen, and D.~Wagner.
\newblock Computing the margin of victory in {IRV} elections.
\newblock Presented at the 2011 Electronic Voting Technology Workshop/Workshop
  on Trustworthy Elections, 2011.

\bibitem[Mathieson and Szeider(2012)]{MS12}
L.~Mathieson and S.~Szeider.
\newblock Editing graphs to satisfy degree constraints: A parameterized
  approach.
\newblock \emph{Journal of Computer and System Sciences}, 78\penalty0
  (1):\penalty0 179--191, 2012.

\bibitem[Mattei et~al.(2012{\natexlab{a}})Mattei, Goldsmith, and
  Klapper]{MGK12}
N.~Mattei, J.~Goldsmith, and A.~Klapper.
\newblock On the complexity of bribery and manipulation in tournaments with
  uncertain information.
\newblock In \emph{Proceedings of the 25th International Florida Artificial
  Intelligence Research Society Conference (FLAIRS~'12)}, pages 549--554. AAAI
  Press, 2012{\natexlab{a}}.

\bibitem[Mattei et~al.(2012{\natexlab{b}})Mattei, Pini, Rossi, and
  Venable]{MPRV12}
N.~Mattei, M.~Pini, F.~Rossi, and K.~Venable.
\newblock Bribery in voting over combinatorial domains is easy.
\newblock In \emph{Proceedings of the 11th International Conference on
  Autonomous Agents and Multiagent Systems (AAMAS~'12)}, pages 1407--1408.
  IFAAMAS, 2012{\natexlab{b}}.

\bibitem[Maushagen et~al.(2018)Maushagen, Neveling, Rothe, and Selker]{MNRS18}
C.~Maushagen, M.~Neveling, J.~Rothe, and A.~Selker.
\newblock Complexity of shift bribery in iterative elections.
\newblock In \emph{Proceedings of the 17th International Conference on
  Autonomous Agents and Multiagent Systems (AAMAS~'18)}, pages 1567--1575,
  2018.

\bibitem[Meir et~al.(2008)Meir, Procaccia, Rosenschein, and Zohar]{MPRZ08}
R.~Meir, A.~Procaccia, J.~Rosenschein, and A.~Zohar.
\newblock The complexity of strategic behavior in multi-winner elections.
\newblock \emph{Journal of Artificial Intelligence Research}, 33:\penalty0
  149--178, 2008.

\bibitem[Monroe(1995)]{Mon95}
B.~Monroe.
\newblock Fully proportional representation.
\newblock \emph{American Political Science Review}, 89\penalty0 (4):\penalty0
  925--940, 1995.

\bibitem[Niedermeier(2006)]{Nie06}
R.~Niedermeier.
\newblock \emph{Invitation to Fixed-Parameter Algorithms}.
\newblock Oxford University Press, 2006.

\bibitem[Procaccia et~al.(2008)Procaccia, Rosenschein, and Zohar]{PRZ08}
A.~Procaccia, J.~Rosenschein, and A.~Zohar.
\newblock On the complexity of achieving proportional representation.
\newblock \emph{Social Choice and Welfare}, 30\penalty0 (3):\penalty0 353--362,
  2008.

\bibitem[Schlotter et~al.(2017)Schlotter, Faliszewski, and Elkind]{EFS17}
I.~Schlotter, P.~Faliszewski, and E.~Elkind.
\newblock Campaign management under approval-driven voting rules.
\newblock \emph{Algorithmica}, 77\penalty0 (1):\penalty0 84--115, 2017.

\bibitem[Sekar et~al.(2017)Sekar, Sikdar, and
  Xia]{sek-sik-xia:c:gehrlein-stable}
S.~Sekar, S.~Sikdar, and L.~Xia.
\newblock Condorcet consistent bundling with social choice.
\newblock In \emph{Proceedings of the 16th International Conference on
  Autonomous Agents and Multiagent Systems (AAMAS~'17)}, pages 33--41. IFAAMAS,
  2017.

\bibitem[Skowron et~al.(2015)Skowron, Faliszewski, and Slinko]{SFS15}
P.~Skowron, P.~Faliszewski, and A.~Slinko.
\newblock Achieving fully proportional representation: {Approximability}
  results.
\newblock \emph{Artificial Intelligence}, 222:\penalty0 67--103, 2015.

\bibitem[Skowron et~al.(2016)Skowron, Faliszewski, and Lang]{SFL16}
P.~Skowron, P.~Faliszewski, and J.~Lang.
\newblock Finding a collective set of items: From proportional
  multirepresentation to group recommendation.
\newblock \emph{Artificial Intelligence}, 241:\penalty0 191--216, 2016.

\bibitem[Xia(2012)]{Xia12}
L.~Xia.
\newblock Computing the margin of victory for various voting rules.
\newblock In \emph{Proceedings of the 13th ACM Conference on Electronic
  Commerce (EC~'12)}, pages 982--999. ACM Press, 2012.

\end{thebibliography}
\newcommand{\bibremark}[1]{}

}

\end{document}